\documentclass[lettersize,journal]{IEEEtran}
\usepackage{amsmath,amssymb,amsfonts}
\usepackage{algorithmic}
\usepackage{algorithm}
\usepackage{array}
\usepackage[caption=false,font=normalsize,labelfont=sf,textfont=sf]{subfig}
\usepackage{textcomp}
\usepackage{stfloats}
\usepackage{url}
\usepackage{verbatim}
\usepackage{graphicx}
\usepackage{cite}
\usepackage{lineno,hyperref}
\usepackage{amssymb}
\usepackage{amsthm}
\usepackage{makecell}
\usepackage{caption}
\usepackage[section]{placeins}
\usepackage{color}

\newtheorem{theorem}{Theorem}
\newtheorem{lemma}[theorem]{Lemma}
\newtheorem{proposition}{Proposition}
\newtheorem{definition}{Definition}

\begin{document}
\title{SSD: Towards Better Text-Image Consistency Metric in Text-to-Image Generation}

\author{Zhaorui Tan$^{1}$\thanks{$^{1}$ Zhaorui Tan, Xi Yang, Zihan Ye, Qiufeng Wang, and Yuyao Yan  are with Xi'an Jiaotong-Liverpool University.},  Xi Yang$^{1, \dag}$, Zihan Ye$^{1}$, Qiufeng Wang$^{1}$, Yuyao Yan$^{1}$, Anh Nguyen$^{2}$\thanks{$^{2}$ Anh Nguyen is with University of Liverpool.}, and Kaizhu Huang$^{3, \dag}$\thanks{$^{3}$ Kaizhu Huang is with Duke Kunshan University.\\ $\dag$: Corresponding authors: Xi Yang (Xi.Yang01@xjtlu.edu.cn) and Kaizhu Huang (kaizhu.huang@dukekunshan.edu.cn)}}



\maketitle

\begin{abstract}
Generating consistent and high-quality images from given texts is essential for visual-language understanding. Although impressive results have been achieved in generating high-quality images, text-image consistency is still a major concern in existing GAN-based methods. Particularly, the most popular metric $R$-precision may not accurately reflect the text-image consistency, often resulting in very misleading semantics in the generated images. Albeit its significance, how to design a better text-image consistency metric surprisingly remains under-explored in the community. 
In this paper, we make a further step forward to develop a novel CLIP-based metric termed as Semantic Similarity Distance ($SSD$), which is both theoretically founded from a distributional viewpoint and empirically verified on benchmark datasets. Benefiting from the proposed metric, we further design the Parallel Deep Fusion Generative Adversarial Networks (PDF-GAN) that aims at improving text-image consistency by fusing semantic information at different granularities and capturing accurate semantics.
Equipped with two novel plug-and-play components:  Hard-Negative Sentence Constructor and Semantic Projection, the proposed  PDF-GAN can mitigate inconsistent semantics and bridge the text-image semantic gap. A series of experiments show that,  as opposed to current state-of-the-art methods, our PDF-GAN can lead to significantly better text-image consistency while maintaining decent image quality on the CUB and COCO datasets.
\end{abstract}

\section{Introduction}


Generating images from text descriptions, usually known as Text-to-Image Generation (T2I), is a challenging task that requires both generating high-quality images and maintaining text-image consistency. Although Generative Adversarial Networks (GANs) based methods~\cite{yuan2019ckd, hong2018inferring, li2020exploring, gou2020segattngan, cheng2020rifegan, ramesh2021zero, tao2020df}  have achieved impressive results in generating high-quality images from text descriptions, they still struggle to keep the text-image consistency within complex semantics. Once the text descriptions become more complex, the semantics of the generated image will likely mismatch its text, although it comes with a high-quality score. Thus, the measurement of text-image consistency remains a critical concern in T2I generation.

Albeit its significance, how to design a better text-image consistency metric surprisingly remains under-explored in the community. To illustrate it clearly, we conduct a simple experiment for current potential text-image consistency metrics in Fig.~\ref{fig:banner}.\footnote{$SSD$, CS and CFID in this paper are scaled by $100$ for better readability. SOA is omitted since it cannot be applied to CUB.}

\begin{figure}[t]
\centering
\includegraphics[width=0.9\columnwidth]{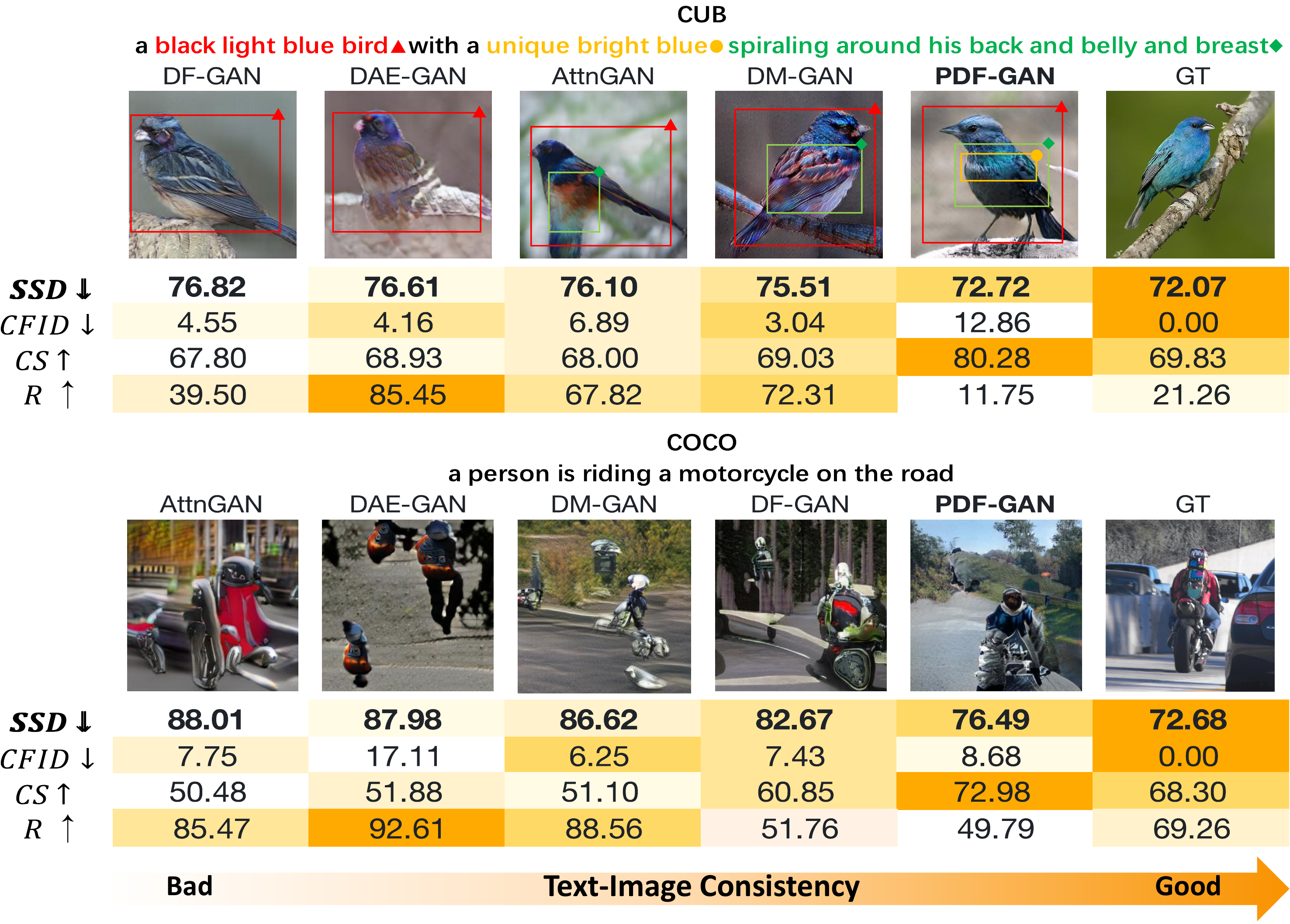}
\caption{Illustration of synthesized images with CFID, CS, $R$, and  $SSD$  score for Attn-GAN, DM-GAN, DAE-GAN, DF-GAN, and our PDF-GAN. Better metric scores are highlighted with higher saturation. $SSD$ demonstrates much better text-image consistency than $R$, CS and CFID.
}
\label{fig:banner}
\end{figure}

\begin{figure}[t]
\scriptsize
\begin{minipage}[t]{0.45\textwidth}
\centering
\includegraphics[width=\textwidth]{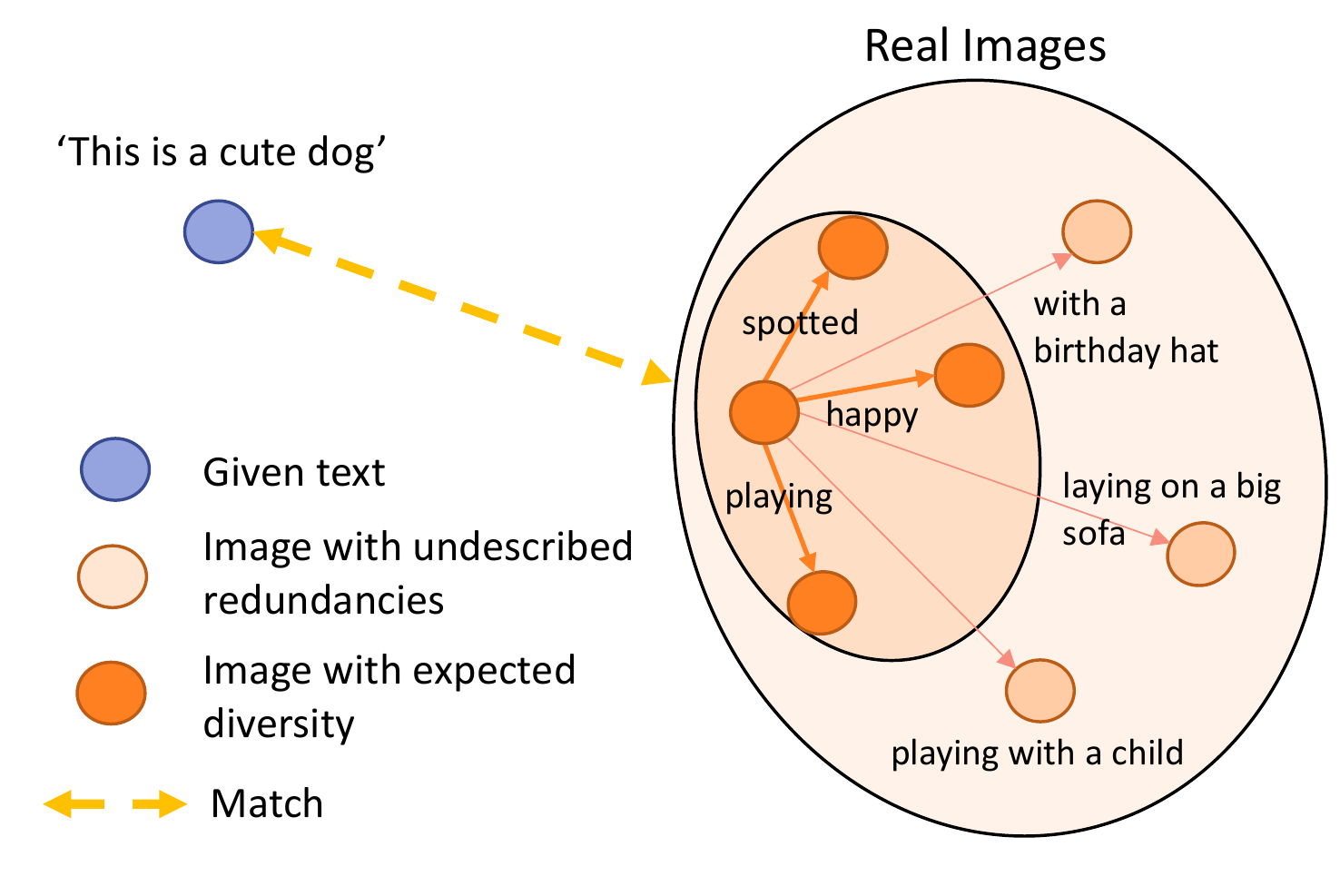}
\caption{Diagram of semantic redundancies. Real images often obtains semantics that may not have been described by given texts. This naturally causes semantic gap. 
}
\label{fig:semantic redundancies}
\end{minipage}
\hfill
\begin{minipage}[t]{0.45\textwidth}
\centering
\includegraphics[width=\textwidth]{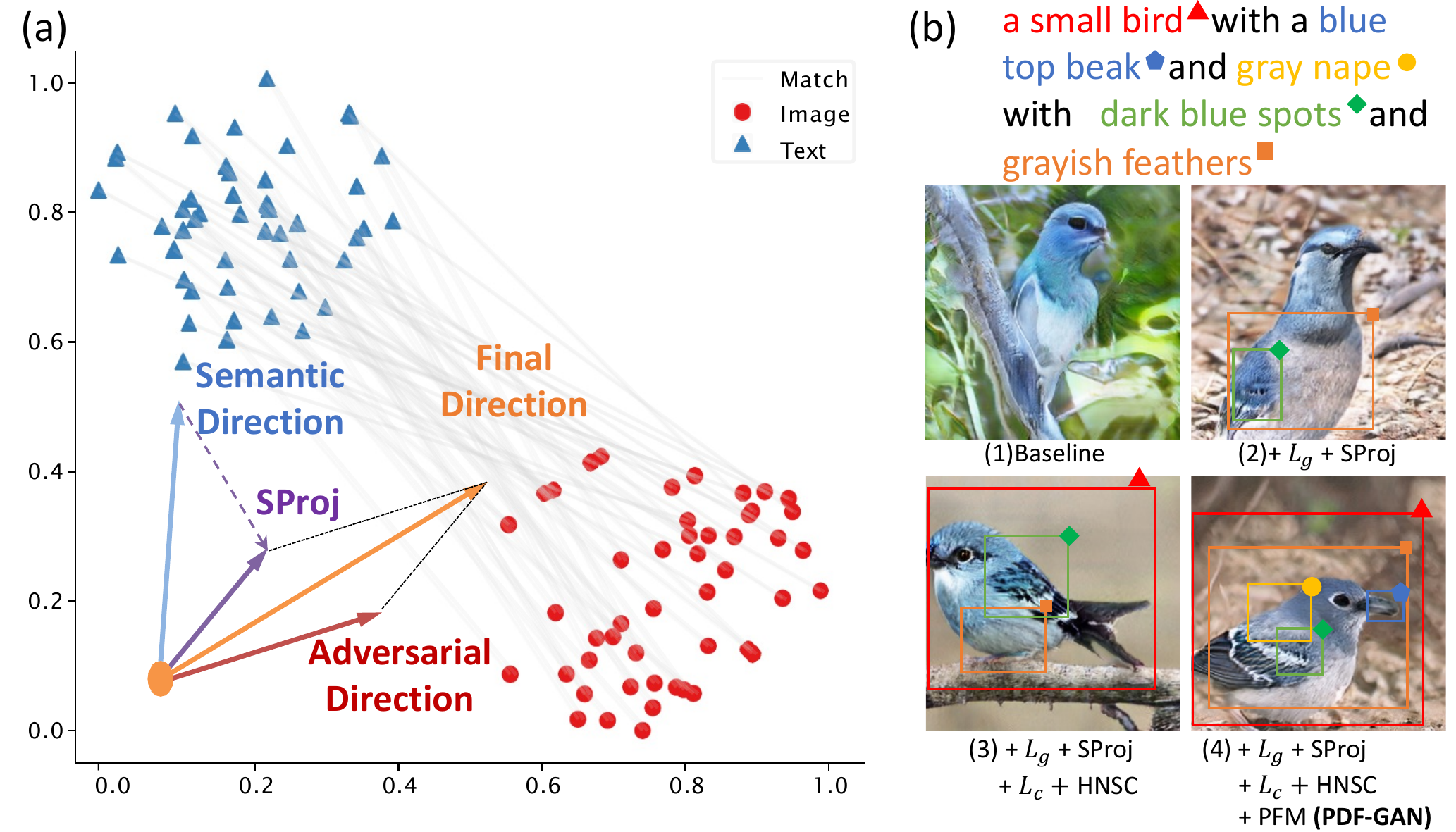}
\caption{
(a) T-SNE map of CLIP embeddings for randomly sampled GT text-image pairs in the CUB dataset. (b) Examples generated of different components are added. }
\label{fig:clipemb}
\end{minipage}
\end{figure}

The most widely used T2I synthesis metric,  $R$-precision ($R$)~\cite{xu2018attngan}, judges text-image consistency by evaluating if the generated image is more consistent with the given text than the other 99 randomly sampled texts. Such a measure may not accurately reflect the direct consistency between texts and images. The right-most column of Fig.~\ref{fig:banner} shows that the $R$ score on ground truth (GT) is even seriously worse than synthetic images.\footnote{We use the 99 random samples that can reproduce the $R$ in AttnGAN~\cite{xu2018attngan}  to calculate $R$ on GT pairs.}  It results in very misleading semantics in the generated image. As highlighted by rectangles in Fig.~\ref{fig:banner}, DAE-GAN achieves the highest $R$, but still produces very text-inconsistent images.  Meanwhile, random sampling may also be highly biased by datasets. The more significant variation between the sampled texts and the given descriptions leads to significantly better scores on more diverse datasets (i.e., COCO~\cite{lin2014microsoft} v.s. CUB~\cite{wah2011caltech} in Fig.~\ref{fig:banner} ($R$)).

Semantic Object Accuracy (SOA)~\cite{hinz2020semantic}, one recently-proposed metric specifically designed for evaluating multi-object text-image consistency,  would still fail to measure the entire semantic consistency without evaluating object attributes and relationships. More seriously, SOA cannot be applied to datasets where only one object usually appears in the generated images, such as CUB.

To alleviate these issues, researchers have to rely on Human Evaluation~\cite{tao2020df}. However, the process is usually costly, and its settings vary a lot among different methods, making it harder to apply in practical scenarios. 

In this work, we make a further step forward
and propose a novel CLIP-based text-image consistency metric from a distributional viewpoint, termed Semantic Similarity Distance ($SSD$). For T2I synthesis tasks, CLIP offers a joint language-vision embedding space where the similarity between semantic distributions of images and text can be directly measured. 
Our $SSD$ is designed by
combining two terms: 1) the first-moment term measures directly the text-image semantic similarity, reflecting the semantic bias between generated images and texts; 2) the second-moment term evaluates the difference of semantic variation between synthesized and real images conditioned on texts, suggesting that the diversity of semantics in the generated images should also be consistent with that of the real images. The second term can bring more credit to precise semantics, balancing the evaluation between overall and detailed consistency. Moreover, due to the large-scale pre-trained CLIP, $SSD$ alleviates the bias in datasets and can be compared across different datasets.

On the theoretical front, we show that $SSD$'s rationale is rooted in using a modified Wasserstein Distance for measuring the divergence of two distributions. We also show that it can be closely linked with two recent metrics, CLIPScore (CS)~\cite{hessel2021clipscore} and Conditional Frechet Inception Distanc (CFID) \cite{soloveitchik2021conditional}, but exhibits more desirable properties in measuring the semantic consistency. CS directly evaluates the  similarity between images and texts' CLIP~\cite{radford2021learning} embeddings, merely accounting for the first term in $SSD$. 
As shown in Fig.~\ref{fig:banner}, generated images may even achieve better CS than GT, partially showing its limitation as a metric.   
Meanwhile, we show that CFID inappropriately measures the distance between text-conditioned real and fake image distributions, which is seriously affected by semantic redundancies (i.e. the semantics not specified by given texts as shown in Fig.~\ref{fig:semantic redundancies}) in real images. Fig.~\ref{fig:banner} also illustrates its incapability of measuring text-image consistency. 

With the benefits of $SSD$, there are two findings as follows: 
1) Different levels of semantic information can significantly help with text-image consistency. However, the semantic gap~\cite{liang2022mind} will cause optimization conflicts between adversarial loss and semantic perceptual loss~\cite{xu2018attngan}, as shown in Fig.~\ref{fig:clipemb} (a). As such, brutally adding semantic perceptual loss weakens the semantic supervision, leading to a sub-optimal performance in text-image consistency.  
2) The mismatched samples for discrimination usually utilize shifted samples in a batch or random samples from other classes, which may lead to degradation of text-image consistency, especially in case of contrastive losses.

According to the above findings, we propose a novel one-stage T2I generation framework named PDF-GAN as Fig.~\ref{framework}, consisting of Parallel Fusion Modules (PFMs) with semantic perceptual losses to fuse different-level granularity textual data. To further improve text-image consistency, we design two novel plug-and-play modules: Hard-Negative Sentence Constructor (HNSC) and Semantic Projection (SProj). HNSC constructs stable and controllable hard negative textural samples instead of sampling mismatched textual samples from the dataset to alleviate dataset bias. SProj constrains the optimization direction of semantic loss, projecting it to the direction that does not conflict with the adversarial loss to overcome the semantic gap. As Fig.~\ref{fig:clipemb} (b) shows, our PDF-GAN with HNSC and SProj significantly improves text-image consistency while maintaining decent image quality. Our code is available at \url{https://github.com/zhaorui-tan/PDF-GAN/tree/main}.

Our contributions are summarized threefold:
\begin{itemize}
    \item We introduce a novel metric, Semantic Similarity Distance, which evaluates both text-image similarity and semantic variation difference between generated images and real images conditioned by the texts. $SSD$ is theoretically well founded and can be cross-compared on different datasets.
    \item We propose a novel framework, Parallel Deep Fusion Generative Adversarial Networks (PDF-GAN), with semantic perceptual losses and PFMs to fuse semantic information at different levels. 
    \item We design an HNSC that mines hard negative textual samples and SProj that alleviates the semantic gap and enhances text-image consistency. 
\end{itemize}

\section{Related Work}

\subsection{GANs for T2I} 
To improve the image quality and size in the first GAN-based approach~\cite{reed2016generative}, most methods adopt a multi-stage architecture for a coarse-to-fine generating process~\cite{zhang2017stackgan, zhang2018stackgan++, wang2021text, xu2018attngan, zhu2019dm}. 
The attention mechanism ~\cite{xu2018attngan, huang2019realistic, ruan2021dae, li2019controllable, gao2021lightweight}  and extra networks~\cite{ma2019sd, zhang2021cross, qiao2019mirrorgan, zhu2019dm, seshadri2021multi, dong2021unsupervised} 
are frequently applied to emphasize the semantics.  
DF-GAN~\cite{tao2020df} proposes a one-stage backbone with Deep Fusion Blocks using Affine and a one-way discriminator output with Matching-Aware Gradient Penalty. It avoids entanglements between generators without using the computationally expensive cross attention.

\subsection{Contrastive Language-Image Pre-training}
CLIP \cite{radford2021learning} is a large-scale multi-modal pre-training model that maps images and language to a joint latent space, aligning them by maximizing their Cosine similarity.
CLIP has been widely used as pre-trained encoders for T2I GAN-based models \cite{brock2018large, gal2021stylegan}, transformer-based generators \cite{wang2022clip}, and diffusion models \cite{ramesh2022hierarchical}.

\subsection{Text-Image Consistency Metrics}
Built upon  the Cosine similarity between image and text embeddings,  
the widely-used metric $R$~\cite{xu2018attngan}  evaluates if the generated image is more similar to the GT texts than random samples from the dataset. 
$R$ does not  measure directly the semantic consistency, which may be highly biased by the dataset. To this end, SOA~\cite{hinz2020semantic} uses a pre-trained object detection model to evaluate whether an  object mentioned by text exists in the generated image.
Failing to measure the entire semantic consistency,
SOA cannot be applied to datasets where only one object appears in the generated images (e.g. CUB).
Owing to CLIP's popularity, CS~\cite{hessel2021clipscore} is designed for image captioning, yet Cosine similarity of CLIP embeddings may not explicitly bind attributes to objects and neglects semantic variations~\cite{ramesh2022hierarchical}. 
With the conditional distribution, CFID~\cite{soloveitchik2021conditional}  evaluates the distance between text-conditioned fake and real image distributions. However, directly aligning fake and real distributions may mismatch the redundant parts in real images, i.e. those contents not specified by texts. This severely affects CFID's effect in measuring text-image consistency.

\section{Semantic Similarity Distance} 
In this section, we set out our novel metric for quantitatively evaluating T2I models. For better measuring the text-image consistency, our metric evaluates not only  direct text-image semantic similarity but also the semantic variation difference between synthesized and real images conditioned on texts.
From a distributional perspective, we assume that normalized embeddings of generated image $\tilde{e}_{f}$, real image $\tilde{e}_{r}$, and text $\tilde{e}_{s}$ distributions are all Gaussian-like distributions $\Phi$ in a joint language-vision embedding space (CLIP space): 
\begin{align}
    \mathbb{Q}_{f} = \Phi (m_f,  & \mathbb{C}_{ff})\;,\;
    \mathbb{Q}_{r} = \Phi (m_r, \mathbb{C}_{rr})\;,\; 
    \mathbb{Q}_{s} = \Phi (m_s, \mathbb{C}_{ss})\;, \;  \notag
\end{align}
where $m$ and $\mathbb{C}$ denote the mean and covariance; $f$, $r$ and $s$ mean the generated images, real images, and texts, respectively.
Conditioned on the same text $s$, the generated and real images's conditional distribution, $\mathbb{Q}_{f|s}, \mathbb{Q}_{r|s}$, are given as:
\begin{align}
    \mathbb{Q}_{f|s} = \Phi (m_{f|s}, \mathbb{C}_{ff|s})\;, \mathbb{Q}_{r|s} = \Phi (m_{r|s}, \mathbb{C}_{rr|s})\;, \nonumber
\end{align}
where $\mathbb{C}_{ff|s}$ and $\mathbb{C}_{rr|s}$ represent conditional covariances of $\tilde{e}_{f}$ and $\tilde{e}_{r}$ that are constant and independent of condition $\tilde{e}_s$.  
We are now ready to define our Semantic Similarity Distance.

\begin{definition} 
As the ultimate goal is to measure the semantic distance among $\tilde{e}_{f}$ and $\tilde{e}_s$, we consider the distance between ${Q}_{f}$, and ${Q}_{s}$, and the distance between ${Q}_{ff|s}$ and ${Q}_{rr|s}$. Our $SSD$ is then defined as follows:
\begin{align}
\small
\label{equ:ssd}
     SSD(\mathbb{Q}_{f},   \mathbb{Q}_{s},\mathbb{Q}_{f|s},\mathbb{Q}_{r|s}) &= 
    \underbrace {[1 - cos( m_f, m_s)]}_{SS}   +\\ 
      &\underbrace {||d(\mathbb{C}_{ff|s}) - d(\mathbb{C}_{rr|s})||^2}_{dSV}  \;,
\end{align}
where $d(\cdot)$ represents matrix's the diagonal part. $SS$ and $dSV$ represents the first and second term of $SSD$ respectively. 
\end{definition}

Since a pre-trained CLIP model is used to map the image and text to a joint language-vision embedding space, it is intuitive to measure their embeddings' Cosine distance, as done in the first-moment term of Eq.~(\ref{equ:ssd}). Due to the semantic gap between ${Q}_{f}$ and ${Q}_{s}$, solely measuring the Cosine distance cannot fully reflect the distribution divergence. 
We then take ${Q}_{f|s}$ and $Q_{r|s}$ into consideration to bridge the semantic gap.
We argue that if a model can fully capture semantics, its generated images should share the same semantic variation as the real images. Semantic variation can also help bind objects and attributes, leading to more precise semantic alignment. Note that we do not align ${Q}_{f|s}$ and $Q_{r|s}$ directly because it over-concerns the redundancies that are not described by the text. 
Therefore, we design a second-moment term in Eq.~(\ref{equ:ssd}) to evaluate the semantic variation by calculating the diagonal differences between text conditioned covariance of fake and real image distributions.

Subsequently, we will support the plausibility of $SSD$ by proving some lemmas.
\begin{lemma}
\label{lemma}
If $\mathbb{C}$ is a non-negative diagonal matrix, the second-moment term can be rewritten as:
\begin{align}
\small
\label{equ:term2}
 &||d(\mathbb{C}_{ff|s}) - d(\mathbb{C}_{rr|s})||^2
\propto\; Tr[(\mathbb{C}_{ff|s}^{\frac{1}{2}} - \mathbb{C}_{ff|s}^{\frac{1}{2}})^2] \;\\
\label{equ:term2_2}
&=\; Tr[\mathbb{C}_{ff|s} + \mathbb{C}_{rr|s} - 2(\mathbb{C}_{ff|s}^{\frac{1}{2}}\mathbb{C}_{rr|s}\mathbb{C}_{ff|s}^{\frac{1}{2}})^{\frac{1}{2}}] \;. 
\end{align}
\end{lemma}

\begin{proof}
According to~\cite{kay1993fundamentals}, conditional covariances can be equivalently written as follows:
\begin{align*}
\mathbb{C}_{ff|s} = \mathbb{C}_{ff} - \mathbb{C}_{fs}\mathbb{C}_{ss}^{-1}\mathbb{C}_{sf}\;,\;
\mathbb{C}_{rr|s} = \mathbb{C}_{rr} - \mathbb{C}_{rs}\mathbb{C}_{ss}^{-1}\mathbb{C}_{sr}\;. 
\end{align*}

$\mathbb{C}_{**}$ is defined as a covariance matrix, which is positive semi-definite.
Meanwhile, in CLIP space we only focus the diagonal part of  $\mathbb{C}$ because CLIP tries to maximize the Cosine similarity between embeddings via training. Thus $\mathbb{C}$ can be simplified as  a non-negative diagonal matrix.  
Therefore for Eq.~(\ref{equ:term2_2}), we have:
\begin{align*}
& Tr[\mathbb{C}_{ff|s} + \mathbb{C}_{rr|s} - 2(\mathbb{C}_{ff|s}^{\frac{1}{2}}\mathbb{C}_{rr|s}\mathbb{C}_{ff|s}^{\frac{1}{2}})^{\frac{1}{2}}] \notag \\
=  & Tr[d(\mathbb{C}_{ff|s}) + d(\mathbb{C}_{rr|s}) - 
2(d(\mathbb{C}_{ff|s})^{\frac{1}{2}}d(\mathbb{C}_{rr|s})d(\mathbb{C}_{ff|s})^{\frac{1}{2}})^{\frac{1}{2}}] 
\notag \;. 
\end{align*}
For the diagonal part, it has:
\begin{align}
&{\sigma_{ff|s}} + {\sigma_{rr|s}} - 2({\sigma_{ff|s}}^{\frac{1}{2}}\  {\sigma_{rr|s}}\ {\sigma_{ff|s}}^{\frac{1}{2}})^{\frac{1}{2}}
\notag \\
=& {\sigma_{ff|s}} + {\sigma_{rr|s}} - 2({\sigma_{ff|s}}\ {\sigma_{rr|s}})^{\frac{1}{2}} 
 \notag \\
=&  [ {\sigma_{ff|s}}^{\frac{1}{2}} - {\sigma_{rr|s}}^{\frac{1}{2}}]^2 \notag 
\propto  [{\sigma_{ff|s}} - {\sigma_{rr|s}} ]^2 \;, \notag
\end{align}
where ${\sigma_{**}}$ is the diagonal part of $C_{**}$. 
Therefore, for Eq.~(\ref{equ:term2_2}) we have:
\begin{align}
&Tr[\mathbb{C}_{ff|s} + \mathbb{C}_{rr|s} - 2(\mathbb{C}_{ff|s}^{\frac{1}{2}}  \mathbb{C}_{rr|s}   \mathbb{C}_{ff|s}^{\frac{1}{2}})^{\frac{1}{2}}]
\notag \\
= \; & Tr[(d(\mathbb{C}_{ff|s})^{\frac{1}{2}} - d(\mathbb{C}_{rr|s})^{\frac{1}{2}})^2] \; \notag
\propto\; ||d(\mathbb{C}_{ff|s}) - d(\mathbb{C}_{rr|s})||^2  \; . \notag
\end{align}
Thus, Eq.~(\ref{equ:term2}) is proportional to Eq.~(\ref{equ:term2_2}). 
\end{proof}

\begin{lemma}
\label{lemma:1}
When $ m_f, m_s $ are normalized, the first-moment term can be rewritten as:
\begin{align}
\small
\label{equ:term1}
 1 - cos(m_f, m_s)   \triangleq \; ||m_f - m_s||^2 \;.
\end{align}
\end{lemma}

\begin{proof}
Cosine distance is equivalent to Euclidean distance of normalized vectors. In CLIP space, $m_f, m_s$ are normalized embeddings of generated image $\tilde{e}_{f}$ and text $\tilde{e}_{s}$.
\end{proof}

We now show how our $SSD$ can be theoretically linked with the other metrics including CS and CFID in T2I synthesis, but exhibits more desirable characteristics.

\begin{proposition} 
If we only use the first term in Eq.~(\ref{equ:ssd}), our $SSD$ is  converted to CLIP-Score (CS).
\end{proposition}

\begin{proof}
The Cosine similarity term $cos(m_f, m_s)$ in Eq.~(\ref{equ:term1}) of $SS$ is equivalent to CS:
\begin{align} 
\small
    cos(m_f, m_s)  &= \mathbb{E}\;[\;cos(\tilde{e}_f,\tilde{e}_s)\;] 
      \propto \omega 
       \notag  \\ &* \mathbb{E}\;[\; \max(\;cos(\;\tilde{e}_f,\tilde{e}_s\;), 0) \;]  = \mbox{CS}(\;\tilde{e}_f,\tilde{e}_s\;), \notag
\end{align}
where $\omega$ is a constant scale-coefficient used in CS.
\end{proof}

\begin{proposition} 
If we measure the distance between $m_{f|s}$ and $m_{r|s}$ for the first term, $SSD$ will be proportional to Conditional Frechet Inception Distance (CFID).
\end{proposition}

\begin{proof}
By Lemma~\ref{lemma} and \ref{lemma:1}, $SSD$  can be rewritten as:
\begin{align} 
\small
\label{equ:ssd_re}
   SSD \propto & \; ||m_f - m_s||^2 
   \notag \\ &+ Tr[\mathbb{C}_{ff|s} + \mathbb{C}_{rr|s} 
    - 2(\mathbb{C}_{ff|s}^{\frac{1}{2}}\mathbb{C}_{rr|s}\mathbb{C}_{ff|s}^{\frac{1}{2}})^{\frac{1}{2}}] \;.  
\end{align}
If we use $m_{f|s}$ and $m_{r|s}$ for the first term, we have:
\begin{align}
\small
\label{equ:ssd_cfid}
    & ||m_{f|s} - m_{r|s}||^2  + \notag \\ 
    & \underbrace {Tr[\mathbb{C}_{ff|s} + \mathbb{C}_{rr|s} - 
      2(\mathbb{C}_{ff|s}^{\frac{1}{2}}\mathbb{C}_{rr|s}\mathbb{C}_{ff|s}^{\frac{1}{2}})^{\frac{1}{2}}}_{TrSV}]\;  \\
     = \; & CFID (Q_{ff|s}, Q_{rr|s})\;.
\end{align}
Then Eq.~(\ref{equ:ssd}) can be changed to CFID. 
\end{proof}

\subsection{Text-based Semantic Similarity Distance}

Since SSD measures the semantic distance between texts and generated images as well as the text-conditioned variation difference between synthesized images and real image semantic distributions, it can also be applied to image captioning tasks subject to  slight modification.  
\begin{definition} 
$SSD$ in Eq.(\ref{equ:ssd}) based on images can be easily altered as $SSD_T$ which is based on texts and can be applied for the image captioning (image-to-text) task:
\begin{align}
\label{equ:ssd_rf}
    SSD_T  = &SSD(\mathbb{Q}_{fs}, \mathbb{Q}_{r}, \mathbb{Q}_{fs|r}, \mathbb{Q}_{s|r}) \\
     = &[1 - cos( m_{fs}, m_r)]   +  || d(\mathbb{C}_{fsfs|r}) -  d(\mathbb{C}_{ss|r})||^2 \;, 
\end{align}
where $fs$ represents fake captions.
\end{definition}
We have conducted a series of experiments as shown in Table~\ref{tab:ssdontext2} on CUB dataset to show that our $SSD_T$ also reflects image-text consistency for image captioning.

\subsection{SSD's Significance}

As shown in Fig.~\ref{fig:ssd_explain}, CS only evaluate the first-moment distance between $m_s$ and $m_f$, which is not enough because the generated image should maintain the shape of $Q_{f|s}$ which is a partial $Q_{r|s}$. CFID only evaluates the first-moment and second-moment distance between $Q_{f|s}$ and $Q_{r|s}$ whose $m_{r|s}$ cannot get rid of undesribed redundancies. Our $SSD$ measures the the first-moment distance between $m_s$ and $m_f$, and the distortion of $Q_{f|s}$ by comparing its second-moment distance with $Q_{r|s}$. 

Our new proposed $SSD$ can be comprehended as evaluating direct consistency between text and images as a first-moment bias term and semantic variation difference between fake images and real images conditioned by text as a second-moment variation term. The first term of SSD is an inversed scaled CS and the second term is a minor correction of the first term that reflects the conditioned semantic distribution distortion of synthesized images, which is similar to the second term of CFID.

In contrast, CS ignores the second-moment variation term, thus causing weakness in estimating semantic variations.
It  neglects that even if two distributions are very close, the semantics distribution in each modality can be distorted. 
CFID inappropriately considers the first-moment bias term as the difference between fake and real image distributions, which may mismatch the redundancy contents in the images. 
Eq.~(1) focuses on measuring primary semantic changes, bringing more consistent attention to the major semantics variation than Eq.~(\ref{equ:ssd_cfid}). Semantic gap hinders CFID measuring between real and generated image semantic distributions. For example, the text `a cute dog' with a real image presents `a cut dog with a birthday hat', and the generated image presents `a running cute dog'. The generated image does not match the real image but matches the text. 
Our experiments also empirically verify the drawbacks of CS, CFID in Table~\ref{tab:semantic res}, and Table~\ref{tab:ssdontext2} where generated and modified results exceed GT on CS and CFID fails on text-image consistency.  

\begin{figure*}[t]
\centering
\scriptsize
\includegraphics[width=\linewidth]{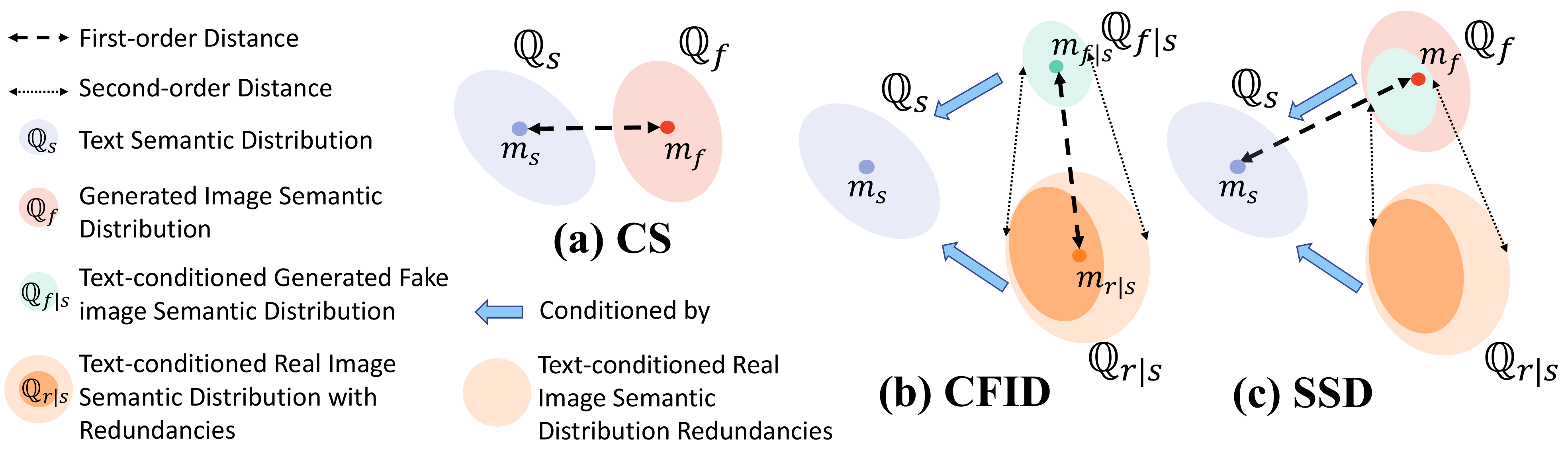}
\caption{ Diagrams of CS, CFID and $SSD$. 
}
\label{fig:ssd_explain}
\end{figure*}

Furthermore, since we use CLIP as encoders rather than random sampling, our metric mitigates the bias in GT data, enabling a convenient comparison across different datasets. As shown in Fig.~\ref{fig:banner}, our $SSD$ indeed reflects better text-image consistency than the other metrics.

\begin{figure*}[t]
\centering
\scriptsize
\includegraphics[width=\textwidth]{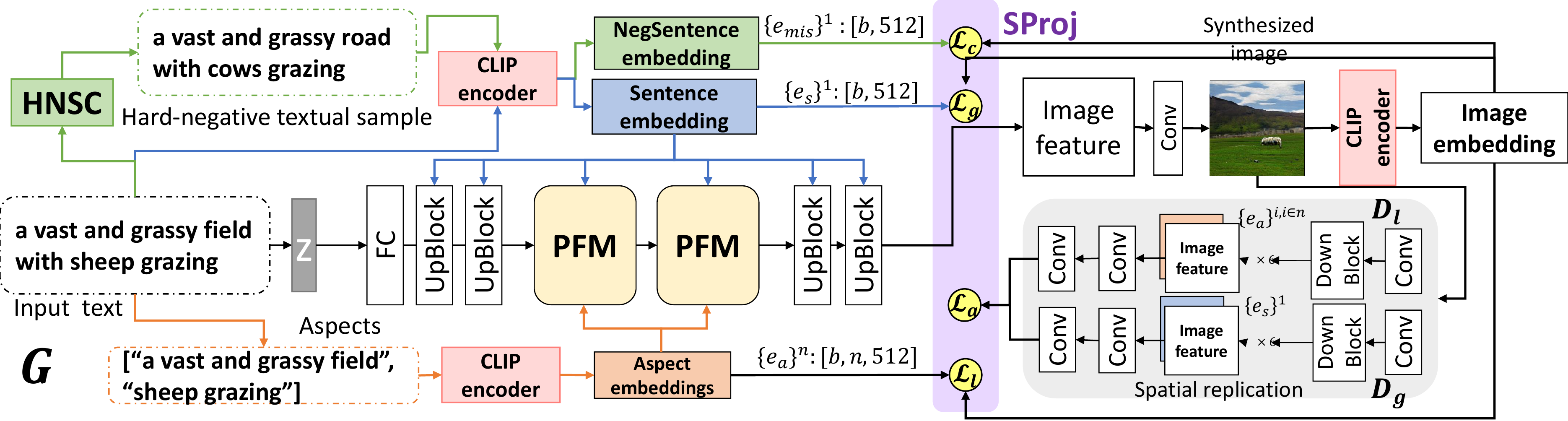} 
\caption{Overall Framework of PDF-GAN.}
\label{framework}
\end{figure*}

\section{Parallel Deep Fusion GAN}
In order to appraise our $SSD$'s feasibility, we further propose PDF-GAN that aims to improving text-image consistency under the guidance of $SSD$.
In this section, we present the Parallel Deep Fusion Generative Adversarial Networks (PDF-GAN) equipped with Hard-negative Sentence Constructor (HNSC) and Semantic Projection (SProj). PDF-GAN fuses semantic information at different levels by using Parallel Fusion Modules (PFM). For semantic supervision, global and local discriminators, semantic perceptual losses, and a contrastive loss are adopted. 
To capture semantic information in texts more precisely and robustly, HNSC creates stable and controllable hard negative samples, and SProj can overcome the semantic gap by constraining the semantic optimization direction. 

\subsection{Parallel Fusion}

In our model, CLIP is used as an encoder to map images and texts into a joint semantic space.
Guided by $SSD$, we confirm that using textual data at different levels improves text-image consistency.
Global-level feature $\{e_g\}^1$ is the caption embedding.
Local-level features $ \{e_l\}^{n}$ are Aspects embeddings which are $n$ key phrases extracted from the caption~\cite{ruan2021dae}. For example, `this black bird has yellow eyes and a long neck', its aspects $\{a\}^{n=3}$ are [`this black bird', `yellow eyes', `long neck']. 
Conditioned by $\{e_g\}^1$ and $\{e_l\}^n$, the generation process is  $x_{f}=G(z \mid \{e_g\}^1, \{e_l\}^{n})\; $ ($G(z)$ for short), where $z\sim N(0,1)$ is a given noise and $x_{f}$ is the generated image. 

In Generator $G$, we propose PFM for efficient fusion between global and local features. PFM takes the output from previous steps as inputs and  $\{e_g\}^1$ and $\{e_l\}^n$ as conditions (in Fig.~\ref{pfm} (c)). The input $h_{t-1}$ is first upsampled to $h_{t-1}'$, then  deep fusion (DF) is conducted on two groups of conditions. After DF, fused features from two branches are concatenated by channels and then go through a convolution layer and outputted as $h_{t}$. 
Especially for the Local DF branch in Fig.~\ref{pfm} (b), we modify the Fusion Block from DF-GAN~\cite{tao2020df} to take local features. In Fig.~\ref{pfm} (a), two groups of MLPs learn the scale and the bias conditioned by local semantics, respectively. Input $h_{t-1}$ is first expanded to the proper shape, then scaled and biased. 
The conditioned features are averaged and passed to later processors. Using PFM to fuse multiple levels of textual information efficiently, our PDF-GAN capture precise semantics while maintaining decent image quality. We use two PFMs in our experiments due to memory limitations.

\begin{figure*}[t]
\centering
\scriptsize
\includegraphics[width=0.85\linewidth]{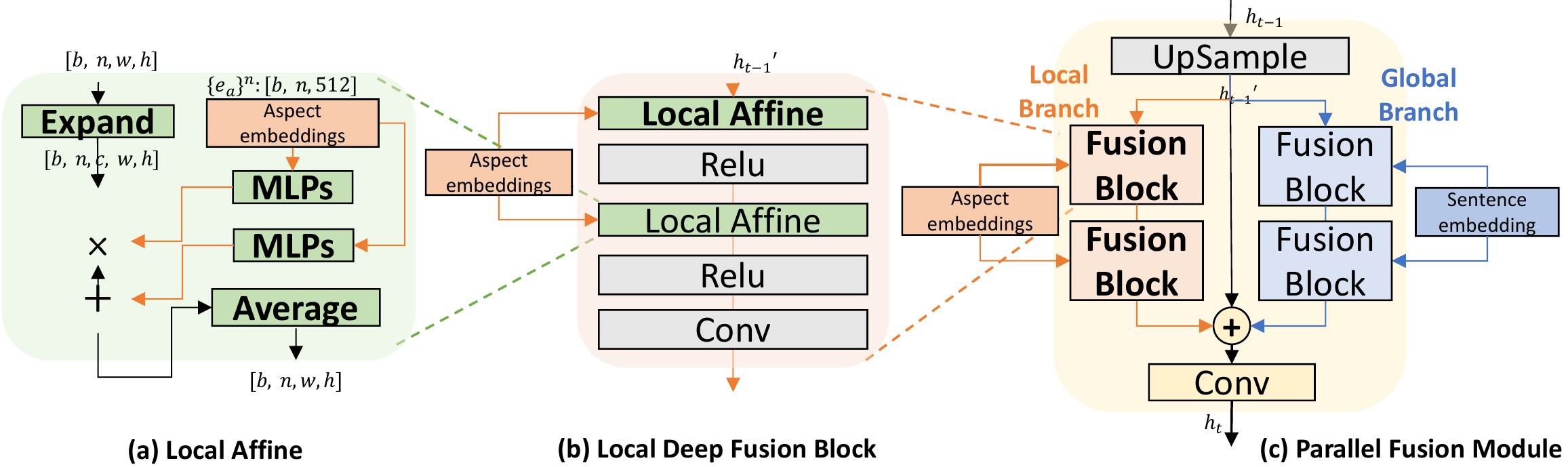} 
\caption{Parallel Fusion Module}
\label{pfm}
\end{figure*}

\subsection{Hard-negative Sentence Constructor}
HNSC constructs hard negative sentence samples by randomly replacing tokens in the given description according to the Part of Speech (POS). Nouns, verbs, and adjectives are replaced by other nouns, verbs, and adjectives.  For example, for the text `this bird is blue on its tail and has a long pointy beak,' our HNSC will randomly replace a certain percentage of words by its POS, (change `blue' to `red,' `tail' to `head,' etc.). Candidates for replacement are gathered from the dataset. 
HNSC produces stable and controllable hard negative textual samples which force discriminators to learn precise semantics. 

\subsection{Training Objectives}

We adopt the hinge loss~\cite{zhang2019self} for Discriminator $D$ and Generator $G$. $D$ usually takes four kinds of pairs: fake image with  real text ($G(z)$, $e$), real image with real text  ($x$, $e$), fake image with mismatched text ($G(z)$, $e_{m}$) and real image with mismatched text ($x$, $e_{m}$). 
To capture semantic information at the global and local levels, we use two discriminators $D_{g}$ and $D_{l}$. For each $D$ given by matched $k$ conditions $\{e\}^k$, and mismatched  $k$ conditions $\{e_{m}\}^k$, the loss with modified Matching Aware Gradient Penalty~\cite{tao2020df} is defined as:
\begin{align} 
\small
\label{equ:Dloss} 
    \mathcal{L}_{D} =\; &\mathbb{E}_{x_r\sim \mathbb{P}_{r}, e \in \{e\}^k}[1 - D_{g}(x,e)] \notag \\+ &\; 
     \frac{1}{2} [ \mathbb{E}_{x_r\sim \mathbb{P}_{r}, e_{m} \in \{e_{m}\}^k}[1 + D_{g}(x,e_{m} ) ] \notag \\+ &\; 
      \mathbb{E}_{G(z) \sim \mathbb{P}_{g}, e \in \{e\}^k}[1 + D_{g}(G(z),e)] ] \notag
     \\+& \;
      q \mathbb{E}_{x\sim \mathbb{P}_{r}}[(||\bigtriangledown_x D_{g}(x,\bar{e})|| + ||\bigtriangledown_{e}D(x,\bar{e})||)^p ]  \;.
\end{align}
The penalized term uses averaged $\{ e\}^k$ as the condition.
Loss functions for $D_g$ and $D_l$
both follow Eq.~(\ref{equ:Dloss}) with that the given conditions $\{e\}^k$ and $\{e_{m}\}^k$ are global level $\{e_g\}^1,\{e_{mg}\}^1 $ and local level ${\{e_l\}}^{n},{\{e_{ml}\}}^{n}$, respectively. Notice that $\{e_{mg}\}^1$ is  hard-negative sentence constructed by HNSC and ${\{e_{ml}\}}^{n}$ are mismatched aspects from the batch. $q$ and $p$ are hyper-parameters set to 2.0 and 6 in our experiments. The adversarial loss of $G$ is defined as:
\begin{align}
\small
      \mathcal{L}_{a} = & -\mathbb{E}_{G(z)\sim \mathbb{P}_{g}}[D_{g}(G(z), \{e_g\}^1)] \notag \\- &\; \mathbb{E}_{G(z)\sim \mathbb{P}_{g}}[D_{l}(G(z), \{e_l\}^n)] \;.
\end{align}
To enhance semantic information at different levels in $G$, we adopt global and local semantic perceptual losses $\mathcal{L}_{g}$, $\mathcal{L}_{l}$: 
\begin{align}
\small
  \mathcal{L}_{g}  = f_{C}(G(\tilde{z}))^T \cdot \tilde{e}_{g}\;, \; \mathcal{L}_{l} =  \frac{1}{n}\sum_{i=1}^{n} f_{C}(G(\tilde{z}))^T \cdot {\tilde{e}_{l}}^i. \notag
\end{align}
A contrastive loss $ \mathcal{L}_{c}$ is introduced to further repel mismatched samples: 
\begin{align}
\small
    \mathcal{L}_{c} & =  \frac{ f_{C}(G(\tilde{z}))^T \cdot {\tilde{e}}_{mg} } {f_{C}(G(\tilde{z}))^T \cdot \tilde{e}_{mg} + f_{C}(G(\tilde{z}))^T \cdot \tilde{e}_g }  \;.
\end{align}
The final generative loss $ \mathcal{L}_{G}$ combines above four losses:
\begin{equation}
\small
\label{equ:Gloss}
    \mathcal{L}_{G} =  \underbrace {\lambda ( \mathcal{L}_{g} +  \mathcal{L}_{l} + \ \mathcal{L}_{c}) }_{Semantic\ loss\  \mathcal{L}_{s}} +  \underbrace {\mathcal{L}_{a}}_{Adversarial\ loss }
\end{equation}
where we set $\lambda = 10$ empirically in our experiments. 

\begin{algorithm}[t]
\small
\caption{Semantic Projection in One Step }
\label{alg:algorithm}
\textbf{Require}: Training data($\{e_g\}^{1},\{e_l\}^{n}$), generator $G$, discriminators $D_l,D_g$, and $z\sim N(0,1)$
\begin{algorithmic}[1] 
    \STATE $x' = G(z \mid \{e_g\}^{1}.\{e_l\}^{n})$
    \STATE $y_g' =  D_g(x' \mid \{e_g\}^{1})$, \; $y_l' =  D_a(x' \mid \{e_l\}^{n})$ 
    \STATE Calculate loss $\mathcal{L}_{D_{g}}$, $\mathcal{L}_{D_{l}}$ for $D_{g}$, $D_{l}$  \quad  { $\triangleright$ See Eq.~(\ref{equ:Dloss}). }
    \STATE ${\delta}_{D_{g}}  \gets  \bigtriangledown \mathcal{L}_{D_{g}}$, ${\delta}_{D_{l}}  \gets  \bigtriangledown \mathcal{L}_{D_{l}}$
    \STATE $\theta_{D_{g}} \gets \theta - \alpha {\delta}_{D_{g}} $, $\theta_{D_{l}} \gets \theta - \alpha {\delta}_{D_{l}} $ 
    
    \STATE Calculate $\mathcal{L}_{a}$, $\mathcal{L}_{s}$  \quad { $\triangleright$ See Eq.~(\ref{equ:Gloss}).}
    
    \STATE ${\delta}_{a} \gets \bigtriangledown \mathcal{L}_{a}$, ${\delta}_{s} \gets \bigtriangledown \mathcal{L}_{s}$
    
    \STATE  $\bar{\delta}_{s} \gets$ \textit{PROJECT}$({\delta}_{a},{\delta}_{s})$  \quad { $\triangleright$ See Eq.~(\ref{equ:SProj}).}
    \STATE $\theta_G \gets \theta_G - (\alpha_a {\delta}_{a} + \alpha_s \bar{\delta}_{s}) $ \\
\textbf{Return}: $x'$ 
\end{algorithmic}
\end{algorithm}

\subsection{Semantic Projection}

Since the semantic gap in CLIP space causes conflicts in optimization directions between $\mathcal{L}_{a}$ and $\mathcal{L}_{s}$, we design SProj to overcome the conflicts (see Fig.~\ref{fig:clipemb}(a)).
Inspired by GEM for continuous learning~\cite{lopez2017gradient}, we treat minimizing $\mathcal{L}_{a}$ and $ \mathcal{L}_{s}$ as two tasks. 
Instead of training on two tasks alternately, we optimize them simultaneously. Algorithm \ref{alg:algorithm} shows the training and protocol of SProj in one step. In each step, after we calculate the gradients ${\delta}_{a}$, ${\delta}_{s}$ for $\mathcal{L}_{a}$ and $\mathcal{L}_{s}$, we conduct \textit{PROJECT} on ${\delta}_{s}$ before we process backpropagation on both tasks. If there is a direction conflict, the semantic optimization direction ${\delta}_{s}$ will be re-projected to a new direction $\bar{\delta}_{s}$ in which it can optimize for $\mathcal{L}_{s}$ while not enlarging $\mathcal{L}_{a}$. \textit{PROJECT} is defined as: 
For gradients $\Delta:= -({\delta}_{a}$, ${\delta}_{s})$, if $ \left \langle {\delta}_{a}, {\delta}_{s} \right \rangle \ge 0$, we solve the Quadratic Program to get solution $\vartheta ^ \ast$ for $\vartheta$: 
\begin{align} 
\small
\label{equ:SProj}
     \min_\vartheta  {1}/{2}\  [\vartheta^T {\Delta}{\Delta}^T + {\delta}_{s}^T {\Delta}^T \vartheta ] \;,  \quad  s.t.  \quad  \vartheta \ge 0\;.
\end{align}
The projected gradients  can then be updated as $\hat{\delta} = G^T \vartheta ^ \ast + g_{s}$.
With SProj, the model can converge to the closest point between texts and image features in the semantic space (as shown in Fig.~\ref{fig:clipemb} (a)), attaining a mutual balance between text-image consistency and image quality.

\section{Experiments}
\label{sec:exp}

We conduct two groups of experiments: 1) The first group of experiments are performed on PDF-GAN to showcase the feasibility of $SSD$ for text-to-image task. It also suggests that our $SSD$ can be used to guide text-to-image model designing. 2)
The second group is a series of simple experiments that show the feasibility of $SSD_T$ for image captioning. These experimental results reflect  that the limitations of the existing text-image consistency metrics  can be alleviated by our $SSD$. 

\subsection{Experimental Setup for PDF-GAN}
In our experiments, we evaluate our metric and method on two datasets, CUB~\cite{wah2011caltech} and COCO~\cite{lin2014microsoft}. The generated images are judged by both text-image consistency metrics and image quality metrics.
%
For qualitatively evaluating the text-image consistency, the proposed $SSD$ is used and compared with other popular metrics $R$~\cite{xu2018attngan}, CS~\cite{hessel2021clipscore}, and CFID~\cite{soloveitchik2021conditional}.  For a fair comparison, CFID is calculated from CLIP embeddings, and all CLIP-based metrics, $SSD$, CS and CFID, are scaled by 100 in our experiment.
The standard metrics, Inception Score (IS)~\cite{salimans2016improved} and Fr{\'e}chet Inception Distance (FID)~\cite{heusel2017gans}, are used to quantitatively evaluate the generated image quality. IS was not used to evaluate COCO because it works not well as indicated in~\cite{tao2020df, zhang2021dtgan}.  Specially, DF-GAN is also trained with CLIP as encoders (DF-GAN+CLIP) by using its original settings for comparison. 

All the metrics are computed over 30K generated images.
We use the released models from competitors for the metric calculation of text-image consistency, and directly adopt their reported image quality results.
The number of the aspect per caption is set to 3, and the maximum number of words per caption is set to 18.

The parameters used for our experiments on CUB and COCO datasets are Table~\ref{tab:paras}. 

\begin{table*}[t]
\centering
\small
\begin{tabular}{c|cc|c|cc}
\hline
\textbf{Parameters}           & \textbf{CUB}      & \textbf{COCO}      &\textbf{Parameters}           & \textbf{CUB}      & \textbf{COCO}     \\ \hline
TRUNCATED\_NOISE     & False    & True    & MAX\_EPOCH           & 600      & 120   \\
MAX\_ATTR\_NUM       & 3        & 3       & DISCRIMINATOR\_LR    & e-4 * 2  & e-4 * 2 \\
MAX\_ATTR\_LEN       & 5        & 5       & GENERATOR\_LR        & 2e-4 / 2 & 2e-4 / 2  \\
BASE\_SIZE           & 64       & 64      & CAPTIONS\_PER\_IMAGE & 10       & 5     \\
BATCH\_SIZE          & 32       & 32      & EMBEDDING\_DIM       & 512      & 512    \\
WORDS\_NUM           & 18       & 18       & & & \\
\hline
\end{tabular}
\caption{Parameters for our experiments for PDF-GAN.}
\label{tab:paras}
\end{table*}

\begin{table*}[t]
\centering
\small
\begin{tabular}{c|c|ccc|c|ccc}
\hline
     & \multicolumn{4}{c|}{\textbf{CUB}}                                           & \multicolumn{4}{c}{\textbf{COCO}}    \\ \cline{2-9}
           & {$SSD\downarrow$}        & CS$\uparrow$              & $R\uparrow$      & CFID$\downarrow$    & {$SSD\downarrow$}        & CS$\uparrow$              & $R\uparrow$      & CFID$\downarrow$  \\ \hline                                             
GT                    & 72.07          & 69.83             & 21.26           & 0.00             & 72.68          & 68.30             & 69.26           & 0.00 \\ \hline                                               
AttnGAN               & 76.10          & 68.00             & 67.82           & 6.89             & 88.01          & 50.48             & 85.47           & 7.75\\                                              
DM-GAN                & 75.51          & 69.03             & 72.31           & \textbf{3.04}    & 86.62          & 51.10             & 88.56           & \textbf{6.25}   \\                                             
DAE-GAN               & 76.61          & 68.93             & \textbf{85.45}  & 4.16              & 87.98          & 51.88             & \textbf{92.61}  & 17.11 \\                                              
DF-GAN                & 76.82          & 67.80             & 39.50           & 4.55              & 82.67          & 60.85             & 51.75           & 7.43 \\
\textbf{PDF-GAN}      & \textbf{72.72} & \textbf{80.28}    & 11.75           & 12.86               & \textbf{75.29} & \textbf{72.71}    & 49.79           & 8.68   \\ \hline                                
\end{tabular}
\caption{Text-image consistency results of $ SSD  $, $CS $, $R $ and $CFID  $ on CUB and COCO. Best results for each metric are highlighted. 
}
\label{tab:semantic res}
\end{table*}

\begin{table*}[t]
\centering
\small
\begin{tabular}{c|ccc|c|ccc|c}
\hline
    & \multicolumn{4}{c}{\textbf{CUB}}        & \multicolumn{4}{|c}{\textbf{COCO}}             \\ \cline{2-9}                                                                       
                 &$SSD \downarrow$                 & $SS \downarrow$                     & {$dSV  \downarrow $} &{$TrSV \downarrow $}  &$SSD\downarrow$                 & $SS\downarrow$                      & {$dSV\downarrow  $} & $ TrSV \downarrow $ \\ \hline
GT         & 72.07                & 72.07                    & {0.00}               & 0.00                  & 72.68              & 72.68                     & {0.00}              & 0.00     \\ \hline
AttnGAN              & 76.10                & 72.80                    & {3.30}               & 0.26                   & 88.01                 & 79.81                     & {8.20}              & 1.26\\
DM-GAN                & 75.51                & 72.39                    & \textbf{{3.12}}      & \textbf{0.19}         & 86.62                 & 79.56                     & {7.06}              & 0.90\\
DAE-GAN             & 76.61                & 72.43                    & {4.18}               & 0.36                    & 87.98                & 79.25                     & {8.73}              & 1.55 \\
DF-GAN              & 76.82                & 72.88                    & {3.94}               & 0.32                     & 82.67               & 75.66                     & {7.01}              & 0.89\\
DF-GAN+CLIP  & 76.63 & 72.00 & 4.63  & 0.11 & 78.75  & 76.53 & \textbf{2.22} & \textbf{0.20}  \\
\textbf{PDF-GAN}      & \textbf{72.72}       & \textbf{67.89}           & {4.83}               & 0.59                  & \textbf{76.49}        &\textbf{70.81}             & {{5.68}}      & {0.60}\\ \hline
\end{tabular}
\caption{Results of $SSD$, $SS$, $dSV$ and $TrSV$. Best results are highlighted as \textbf{bold}. } 
\label{tab:dsv-trsv}
\end{table*}

\subsection{Experimental Setup for Text Modification}

For testing $SSD_T$, we modify the original texts from CUB dataset and use modified texts as toy examples for image captioning task. 
The group of experiments ``Cut Words" means that we slice original captions and take them as fake captions for Eq.~(\ref{equ:ssd_rf}) calculation and $[*:*]$ indicates the slice we take. The group of experiments ``Replaces Words" mean that we replace words in the caption and use them as fake captions. We engage our HNSC to produce the replaced captions. Parameter $r$ means the ratio of the caption that we replace. We compare our $SSD_T$ with image captioning metric CLIPScore~\cite{hessel2021clipscore} (CS).
 We also deploy 30K text-image pairs for calculating $SSD_T$ and CS.

\subsection{Appraise $SSD$}

We appraise our $SSD$ by comparing it with
CS, $R$ and CFID on AttnGAN~\cite{xu2018attngan}, DM-GAN~\cite{zhu2019dm}, DAE-GAN~\cite{ruan2021dae}, DF-GAN~\cite{tao2020df}, DF-GAN+CLIP and our PDF-GAN. Quantitative results are presented in Table~\ref{tab:semantic res}~\footnote{$SSD$, $SS$, $dSV$ and $TrSV$ are by $100$ for better readability in this paper.}, and the generated examples are visualized in Figs.~\ref{fig:bird examples}-\ref{fig:coco examples}. 

\begin{figure}[t]
\centering
\includegraphics[width=0.9\linewidth]{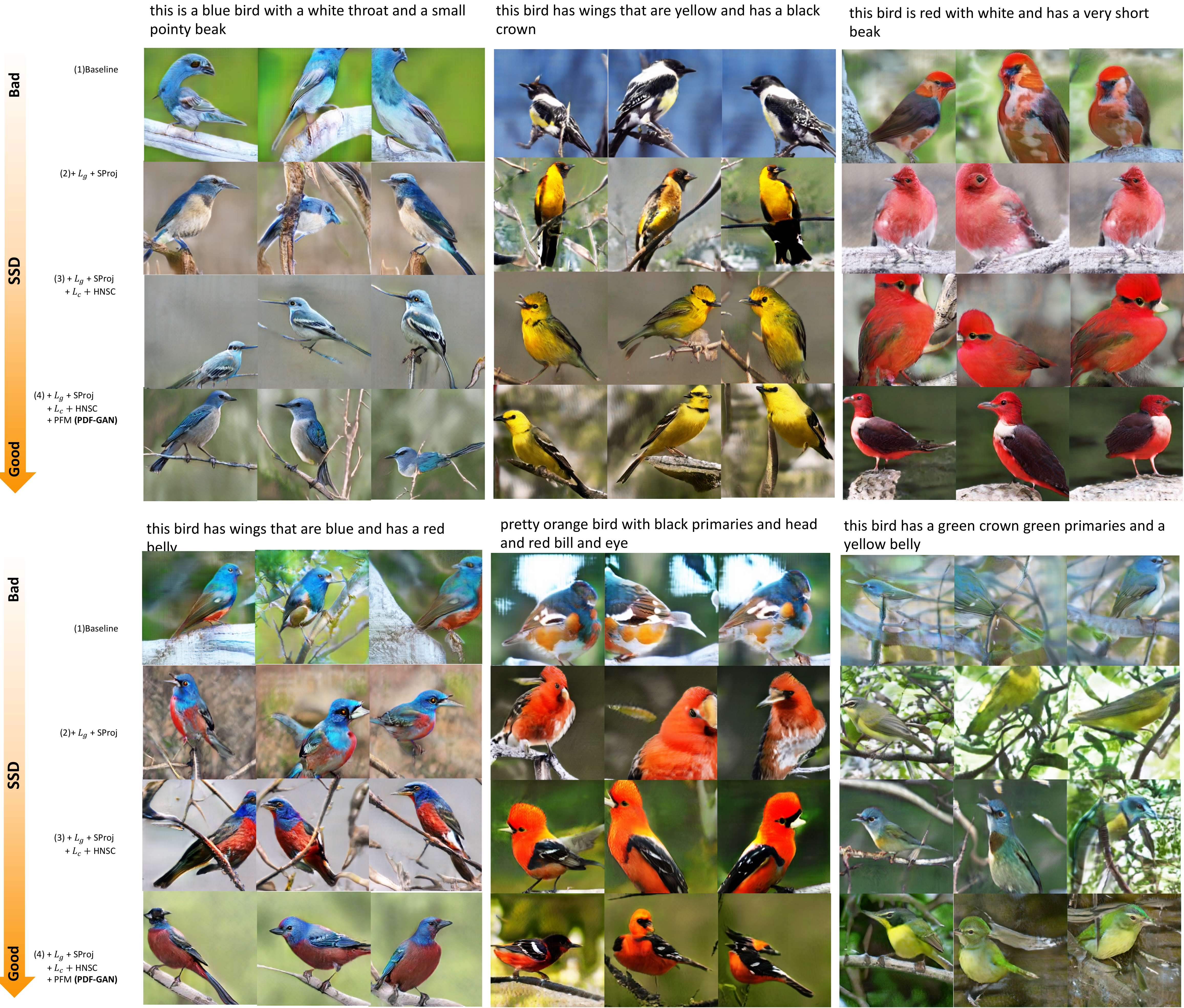} 
\caption{Generated examples of ablation studies. Generated examples organized as $SSD$-descending order--Continuous. Note that ablation task +$\mathcal{L}_{c}$ (CS: 71.76, $SSD$: 74.49) obtains better CS but worse $SSD$ than +HNSC (CS: 71.08, $SSD$: 73.68). }
\label{fig:ablation examples}
\end{figure}

\begin{table*}[t]
\centering
\small
\begin{tabular}{m{0.9cm}<{\centering}|m{0.8cm}<{\centering}m{0.6cm}<{\centering}m{0.7cm}<{\centering}|m{0.6cm}<{\centering}m{1cm}<{\centering}|m{0.8cm}<{\centering}|m{0.8cm}<{\centering}m{0.6cm}<{\centering}m{0.7cm}<{\centering}|m{0.6cm}<{\centering}m{1cm}<{\centering}}
\hline
 \textbf{Cut}  & $SSD_T\downarrow$              & $SS\downarrow$        & $dSV\downarrow$                     & CS$\uparrow$              & CFID           $\downarrow$     & \textbf{Rep} & $SSD_T\downarrow$              & $SS\downarrow$        & $dSV\downarrow$                     & CS$\uparrow$              & CFID $\downarrow$   \\ \hline 
GT                      & 72.08            & 72.08     & 0.00                    & 69.80             & 0.00    & GT                      & 72.08            & 72.08     & 0.00                    & 69.80             & 0.00    \\ \hline
  {[}1:{]}                & \textbf{72.81}   & 71.55     & 1.26                    &\textbf{71.13}     & 2.01    & r=0.05                  & 78.46            & 73.77     & 4.69                    & 65.58             & 9.22 \\
  {[}1:-1{]}              & \textbf{72.92}   & 71.61     & 1.31                    & \textbf{70.98}    & 2.12     & r=0.1                   & 80.39            & 74.04     & 6.35                    & 64.90             & 10.66   \\
  {[}2:{]}                & 76.05            & 72.50     & 3.55                    & 68.75             & 2.01     & r=0.6                   & 92.17            & 76.43     & 15.74                   & 58.93             & 25.93   \\
  {[}3:{]}                & 79.03            & 72.84     & 6.19                    & 67.90             & 7.78     &&&&& \\ \hline

\end{tabular}
\caption{Results of `Cut Words' (`Cut' at right-hand side) and `Replace Words' (`Rep' at left-hand side) experiments of $SSD_T$ for CUB.  $SSD_T$, $dSV$, $SS$, and CS are scaled by $100$ for readability. Results exceed GT are highlighted as \textbf{bold}. }
\label{tab:ssdontext2}
\end{table*}

As shown in Table~\ref{tab:semantic res}, $R$'s indirect measurement and highly-dataset biased sampling generate even better scores for most methods than GT, indicating its severe limitations.
Leveraging  CLIP embeddings, CS can better reflect text-image consistency. However, since CLIP may suffer from binding attributes to objects~\cite{ramesh2022hierarchical}, CS struggles to reflect precise semantics consistency. E.g., it might lead that our PDF-GAN exceeds GT in CS. Such drawback of CS can also be seen in Fig.~\ref{fig:ablation examples} 
from our ablation studies, where the model failing to capture precise semantics like `orange and white' but obtains higher CS. It can also been seen in Table.~\ref{tab:ssdontext2} where modified captions exceed CS on real captions.

CFID aims to measure the distance between images' distributions approximately. 
If we compare the CFID scores in Table~\ref{tab:semantic res} with the corresponding examples in Fig.~\ref{fig:bird examples}, Fig.~\ref{fig:coco examples}, it is evident that CFID is not in line with the consistency. This is mainly because real images usually contain redundancies not mentioned by texts which may however be not reacquired in the generated images, causing unsuitable fist-moment term in Eq.~(\ref{equ:ssd_cfid}). 

Our $SSD$ mitigates the drawbacks of both CS and CFID. 
As shown in Fig.~\ref{fig:bird examples} and Fig.~\ref{fig:coco examples} that are organized in $SSD$ descending-order horizontally,  text-image consistency can indeed be observed   better from left to right. In contrast to CS in Fig.~\ref{fig:ablation examples}, our $SSD$ demonstrates clearly better semantic consistency. Furthermore, since our $SSD$ alleviates the dataset bias, we can compare $SSD$ across different datasets. 

\begin{figure}[htbp]
\centering
\includegraphics[width=0.9\linewidth]{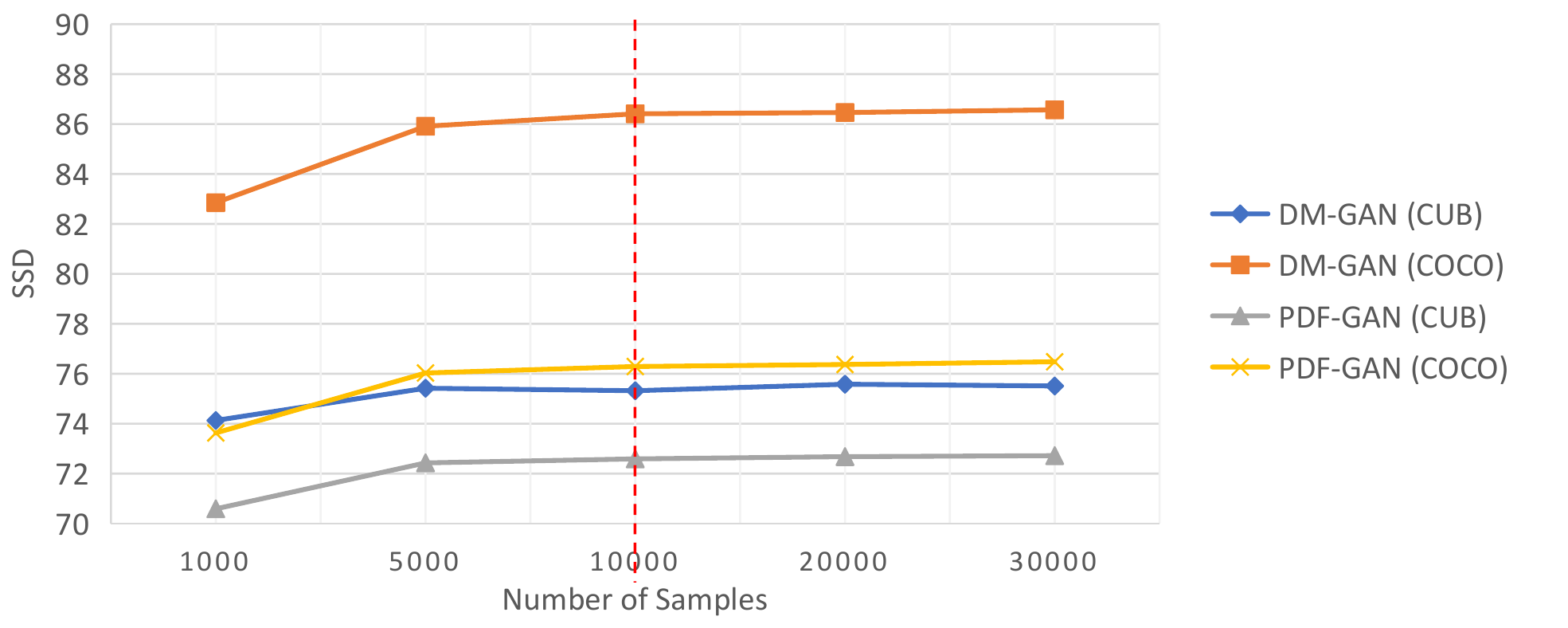}
\caption{$SSD$ of DM-GAN~\cite{zhu2019dm} and PDF-GAN vs. different number of samples.}
\label{fig:metric_stable}
\end{figure}

\subsubsection{$SS$ and $dSV$}

We show that $dSV$ has a significant contribution when $SS$ scores are similar. Since the other methods may lead to high semantic bias, we take the results from our ablation studies for a better demonstration. Quantitative results of ablation studies are shown in Table~\ref{tab:ablationsetting} and qualitative results are shown in Fig.~\ref{fig:ablation examples}.

 

It can be observed that when our $SSD$  is reduced, the generated images tend to contain more consistent semantics with the given texts. Fig.~\ref{fig:ablation examples}(1) does not obtain good text-image consistency due to  high semantic bias even with a low $dSV$. When $SS$ achieves a better level of about $69$,  $dSV$ makes a difference. 
As shown in Fig.~\ref{fig:ablation examples}(2)-(4), even with a similar $SS$, the models with better $SSD$ can lead to significantly better text-consistent images. Note that task (3) with higher CS but lower $SSD$ and (4) with lower CS and higher $SSD$ in Fig.~\ref{fig:ablation examples}, task (4) are more consistent with the given texts, indicating the importance of the correction brought by $dSV$.


\begin{figure*}[t]
\centering
\scriptsize
\includegraphics[width=0.9\textwidth]{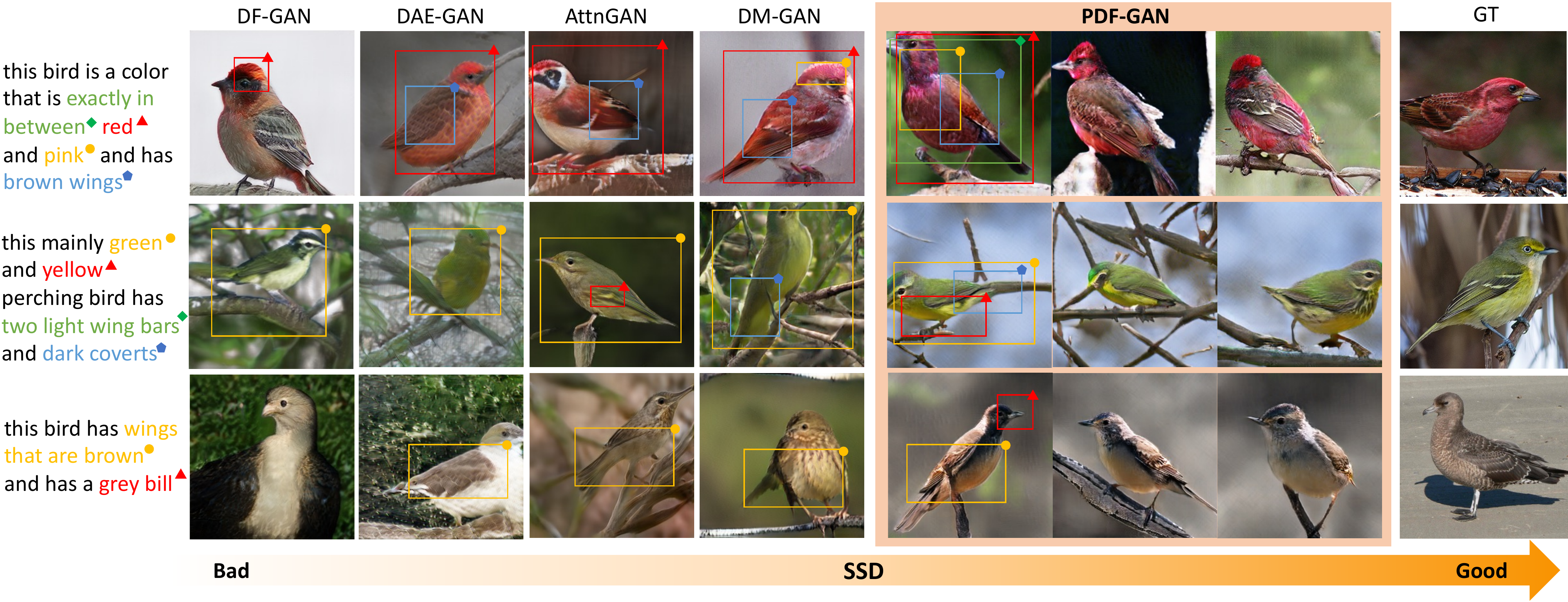}
\caption{Examples for T2I synthesis by AttnGAN, DM-GAN, DAE-GAN, DF-GAN and our PDF-GAN on CUB.}
\label{fig:bird examples}
\end{figure*}

\subsubsection{$dSV$ Vs. $TrSV$}
\label{sec:ssdthancfid}

We show additional experiments to support that our $SSD$'s second-moment term ($dSV$) is more indicative of the major semantic variations than $CFID$'s second-moment term ($TrSV$). Notice that we exploit the diagonal part of $\mathbb{C}$ to calculate $TrSV$. We also display the first-moment term ($SS$) of $SSD$ for better readability.
In Table~\ref{tab:dsv-trsv}, it clearly shows that $TrSV$ and $dSV$ have the same tendency to reflect semantic changes, as they are equivalent. But, $TrSV$ cannot obviously represent semantic changes, while our $dSV$ can produce more significantly different indications.  

\subsubsection{Appraise $SSD_T$}

 Table~\ref{tab:ssdontext2} exhibits results of $SSD_T$
In the first group of experiments, when we cut out the first word and both first and last words, they have better CS scores than ground truth (GT),  suggesting that CS is quite limited. Our $SSD_T$'s $dSV$ can overcome this issue. It can be seen that our $SSD_T$ gets worse as the modification progresses. CFID gets the same tendency as $SSD_T$ because there are quite fewer redundancies between texts, i.e., the images usually contain all textual described information. This also indicates that the text-image consistency evaluation is more difficult than image-text consistency evaluation.

\begin{figure}[t]
\small
\centering
\includegraphics[width=0.95\columnwidth]{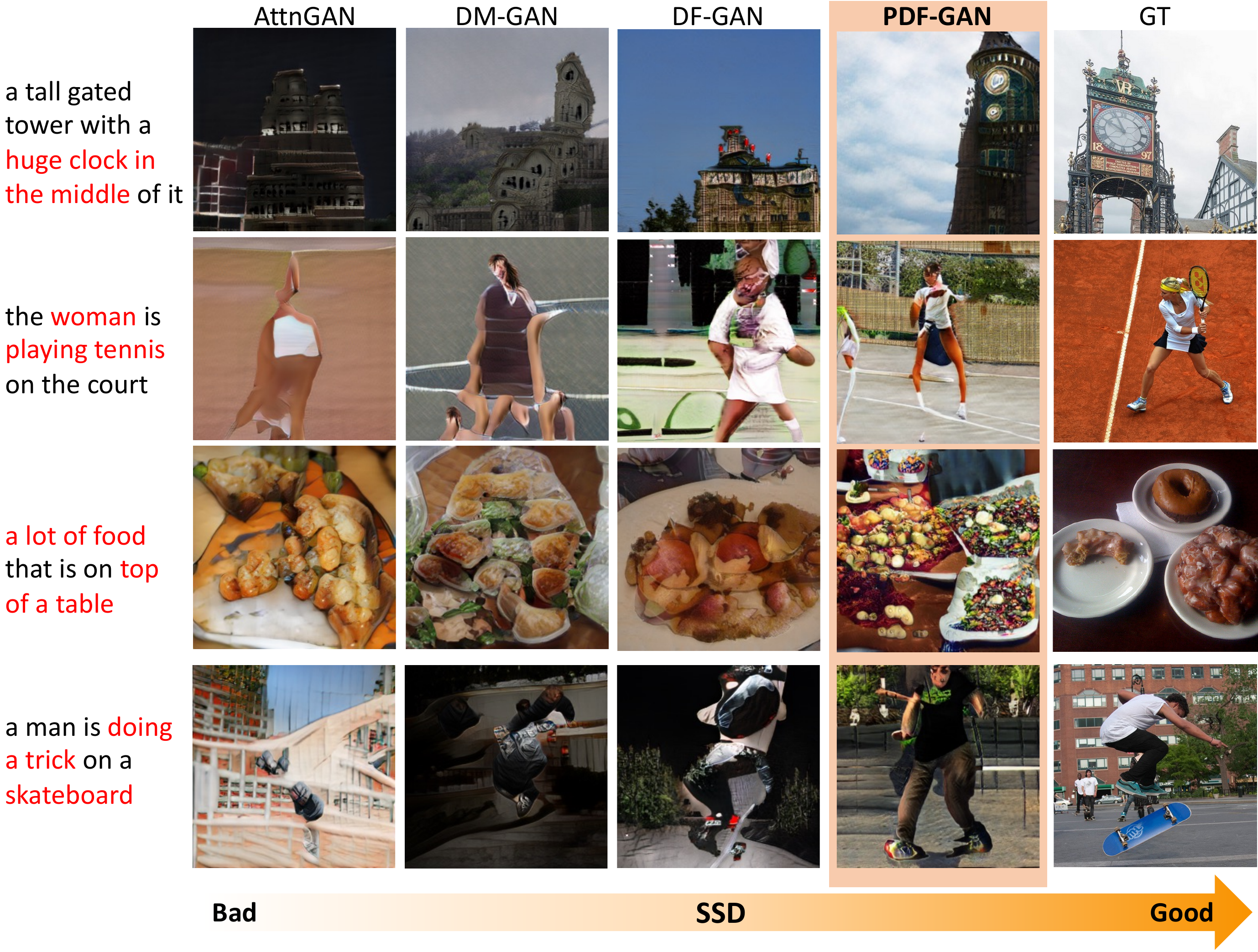} 
\caption{Examples for T2I synthesis by AttnGAN, DM-GAN, DF-GAN and our PDF-GAN on COCO.}
\label{fig:coco examples}
\end{figure}

\begin{table}[t]
\centering
\begin{tabular}{c|cc|c}
\hline
\multicolumn{1}{l|}{} & \multicolumn{2}{c|}{\textbf{CUB}}           & \textbf{COCO}              \\ \cline{2-4} 
                      &IS   $ \uparrow $ & FID $ \downarrow $ & $FID \downarrow  $ \\ \hline
AttnGAN               & 4.36           & 23.98             & 35.49             \\
DM-GAN                & 4.75           & 16.09             & 32.64             \\
DAE-GAN               & 4.42           & 15.19             & 28.12             \\
MirrorGAN             & 4.56           & 18.34             & 34.71             \\
SD-GAN                & 4.67           & -                 & -                 \\
TIME                  & \underline{4.91}     & \underline{14.30}             & 34.14       \\
DF-GAN                & \textbf{5.10}  & 14.81             & \textbf{19.32}    \\

DF-GAN + CLIP         & 4.63        & 24.16                 & 21.84     \\
\textbf{PDF-GAN}      & 4.59           & \textbf{12.30}    & \underline{21.01}       \\ \hline
\end{tabular}
\caption{$IS$ on CUB and $FID$ on CUB and COCO. Best and second best results are highlighted as \textbf{bold} and by \underline{underlines}, respectively.}
\label{tab:image_res}
\end{table}


\subsubsection{Stability of $SSD$}
\label{sec:ssdstable}
As shown in Fig.~\ref{fig:metric_stable} we use different numbers of samples for $SSD$ calculation of DM-GAN~\cite{zhu2019dm} and our PDF-GAN. It turns out that our $SSD$ becomes stable after more than 10K samples are used. Thus we recommend using more than 10K samples for the $SSD$ calculation. 

\begin{table*}[t]
\centering
\small
\begin{tabular}{l|m{0.7cm}<{\centering}m{1.2cm}<{\centering}m{0.7cm}<{\centering}m{0.7cm}<{\centering}m{0.7cm}<{\centering}m{0.7cm}<{\centering}m{1.2cm}<{\centering}|ccc|c|cc}
\hline
Task                            & {CLIP} & {$D_{g}$ + $ \mathcal{L}_{g}$ }      & {SProj}        &   {$\mathcal{L}_{c}$ \;\;}  &  {HNSC}  & {PFM}  & {$D_{l}$ + $ \mathcal{L}_{l}$ }  & {$ SSD \downarrow $}     &{$SS\downarrow$ }     &{$dSV\downarrow$  }            & {CS  $\uparrow $ }     &{IS $\uparrow $ }        & {FID $\downarrow $ } \\ \hline
Baseline                        &\checkmark    &                &                   &                &                  &           &          & 76.63                  &72.00              &\textbf{4.63}              & 59.94                & 4.50                  & 24.16                                                    \\
+  $\mathcal{L}_{g}$            &\checkmark    &\checkmark      &                   &                &                  &           &           & 75.59                  &69.55             &6.04                        & 68.10                & 4.63                  & 14.69                                                  \\
+ SProj                         &\checkmark    &\checkmark      &\checkmark         &                &                  &           &           & 74.10                  &68.77             &5.33                        & 70.70                & \textbf{4.81}         & 15.55                                                  \\ \hline
+  $\mathcal{L}_{c}$            &\checkmark    &\checkmark      &\checkmark         &\checkmark      &                  &           &             & 74.49                  &68.45             &6.04                        & 71.76                & 4.58                  & 14.78                                                 \\
+ HNSC                          &\checkmark    &\checkmark      &\checkmark         &\checkmark      &\checkmark        &           &              & 73.68                  &68.67             &5.01                        & 71.03                & 4.61                  & 14.41                                               \\ \hline
\makecell[l]{+ PFM \\ \textbf{PDF-GAN}  }      &\checkmark    &\checkmark      &\checkmark         &\checkmark      &\checkmark        &\checkmark &\checkmark      & \textbf{72.72}         &\textbf{67.89}    &4.83                        & \textbf{73.63}       & 4.59                  & \textbf{12.30}                          \\ \hline
\end{tabular}
\caption{Ablation settings and ablation study results of different components on CUB. All components are added cumulatively as shown in Table~\ref{tab:ablationsetting}.  Note experiment $+L_{c}$ uses random samples for $L_{c}$ calculation. }
\label{tab:ablationsetting}
\end{table*}

\begin{figure*}[t]
\centering
\begin{minipage}{0.48\textwidth}
\centering
    \scriptsize
    \includegraphics[width=\linewidth]{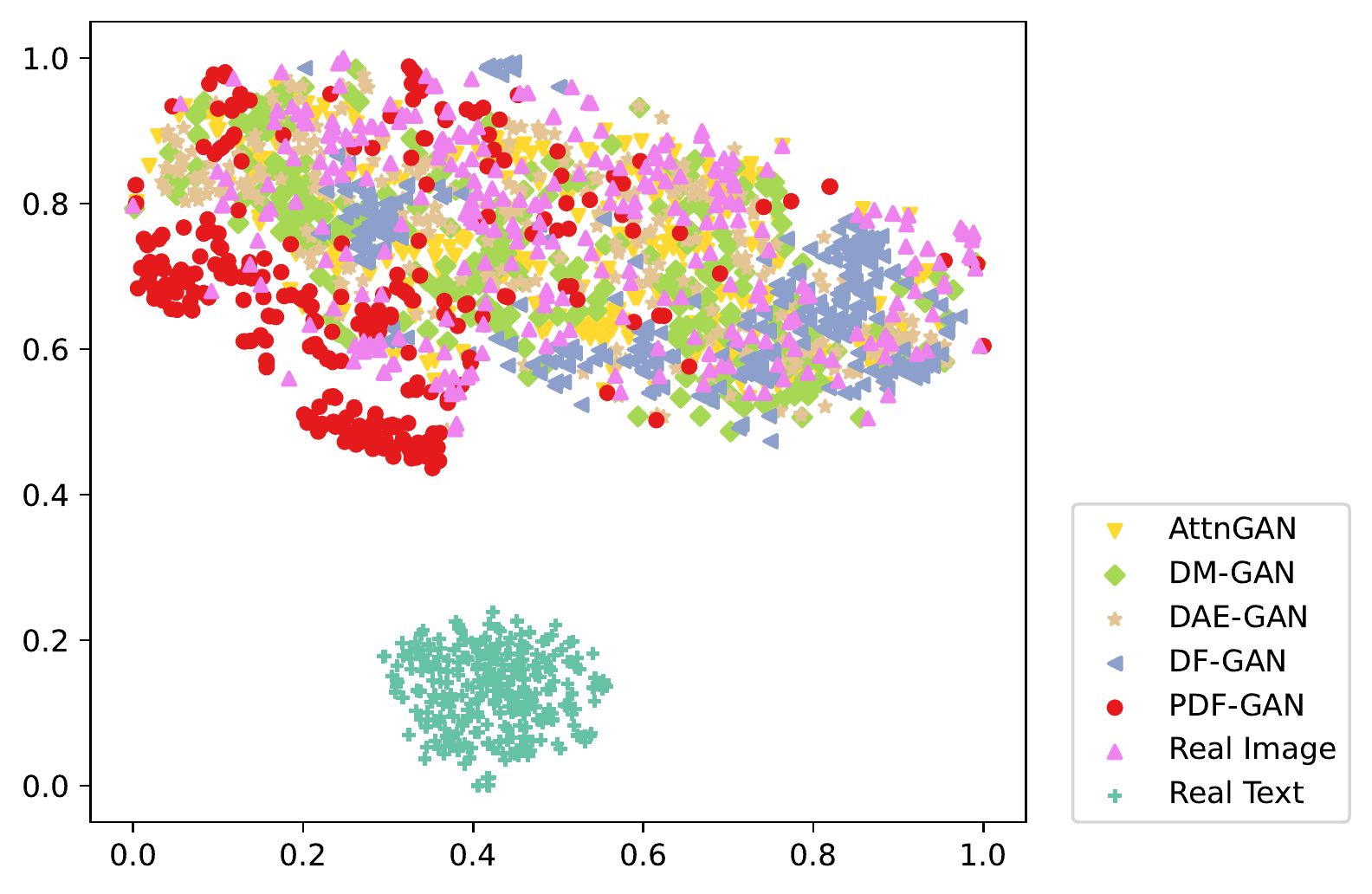}
    \caption{TSNE map of CLIP embeddings of generated images by AttnGAN, DM-GAN, DAE-GAN, DF-GAN, and our PDF-GAN compared with real images and texts on CUB. }
    \label{fig:bird_sd}
\end{minipage}
\hfill
\begin{minipage}{0.48\textwidth}
  \centering
    \scriptsize
    \includegraphics[width=\linewidth]{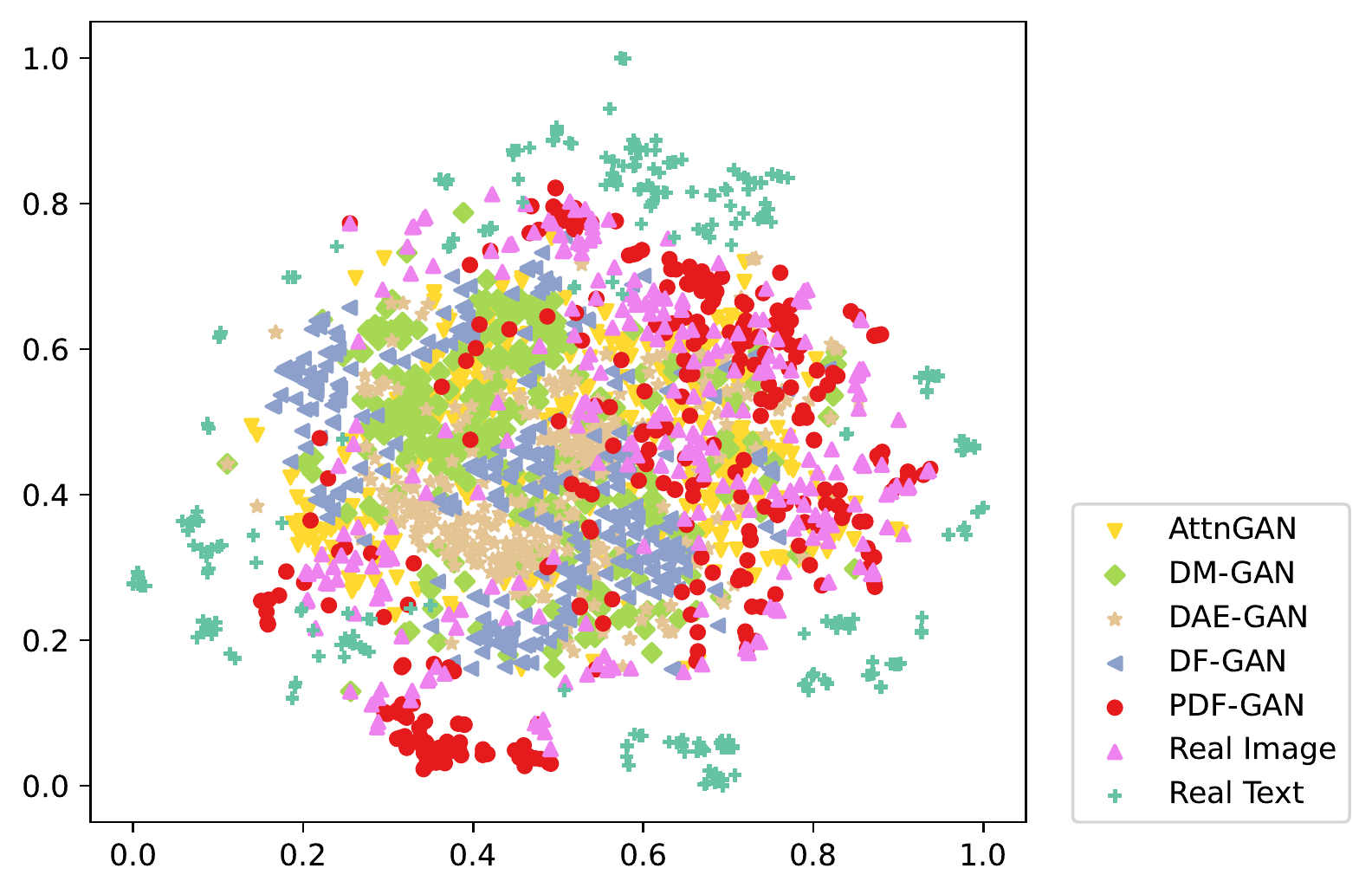}
    \caption{TSNE map of CLIP embeddings of generated images by AttnGAN, DM-GAN, DAE-GAN, DF-GAN, and our PDF-GAN compared with real images and texts on COCO. }
    \label{fig:coco_sd}
\end{minipage}
\end{figure*}

\subsection{Appraise PDF-GAN}

If we examine $SSD$ scores in Table~\ref{tab:semantic res}, all the methods show better text-image consistency on CUB than that on COCO, while our PDF-GAN achieves the best text-image consistency. Meanwhile, PDF-GAN can maintain decent image quality, achieving the best ($12.30$) and second best ($21.01$) in FID on CUB and COCO, respectively. 

As shown in Figs.~\ref{fig:bird examples}-\ref{fig:coco examples}, PDF-GAN can capture complex and precise semantics while maintaining decent quality.
For a less complex dataset like CUB, 
our PDF-GAN surpasses $2.80$ in $SSD$ over the second best method while DF-GAN produces high-quality but text-mismatched images; other attention-based methods fail to capture detailed semantics. Moreover, PDF-GAN can stably generate complex semantics in visual (e.g. in Fig.~\ref{fig:bird examples} where the bird matches  `exactly in red and pink' and other highlighted phrases).  
The superiority of our PDF-GAN is more evident on COCO. 
In comparison, DF-GAN and DF-GAN+CLIP have more balanced results on both datasets and may also produce high-quality images, but the overall text-image consistency is not promising. Our proposed PDF-GAN can capture the complex semantics and have $6.18$ improvement in $SSD$ on COCO, suggesting our model is able to capture complex semantics (see Fig.~\ref{fig:coco examples}). 

For better visualization, we illustrate semantic distributions of fake images generated by our PDF-GAN and other models by comparing with real images and texts as Fig.~\ref{fig:bird_sd} and Fig.~\ref{fig:coco_sd}. Our generated images obtain the semantic distribution closer to given texts than other methods. 

\subsubsection{Ablation Study}
We take CUB as one typical example for ablative analysis. Our ablation study settings and results can be seen in Table~\ref{tab:ablationsetting} whilst examples are shown in Fig.~\ref{fig:clipemb} (b).
Two findings are verified: 
1) The semantic gap can cause optimization conflicts between adversarial and semantic losses. 2) Using random samples from the dataset or the training batch may prevent discriminators from learning precise semantics.
Overall, these two issues could be well tackled by the proposed SProj and HNSC. 

Reported in the first group of experiment of Table~\ref{tab:ablationsetting}, brutally adding semantic perceptual loss can improve image quality, but may not significantly improve text-image consistency due to the semantic gap. Specifically, the decrease of $SS$ is accompanied by a significant increase of $dSV$, indicating when the semantic distribution of generated images gets closer to the texts, its distribution form is likely to be distorted.  
SProj can strengthen the $\mathcal{L}_{g}$'s supervision, forcing the model to be further optimized in semantics without sacrificing the image quality. As shown in Fig.~\ref{fig:clipemb} (b)(2), SProj with semantic perceptual loss benefits capturing more semantic details, easing semantic distribution distortion of generated images. 

The second group shows that contrastive loss $\mathcal{L}_{c}$ may deteriorate text-image consistency if random mismatched texts are given.
The scores of $SSD, dSV$, and IS get worse when this loss is added.
However, HNSC can mitigate this issue by offering stable and controllable hard-negative sentences as mismatched samples. It pushes the discriminators to learn more semantic differences between positive and hard negative texts, thus improving the model's capability in capturing precise semantics through reducing semantic distribution distortion of generated images, as seen in Fig.~\ref{fig:clipemb} (b)(3) and Fig.~\ref{fig:ablation examples}. It also eases the degradation in IS and improves FID.

\begin{table}[htbp]
\centering
\small
\begin{tabular}{c|cc|cc}
\hline
                            & $SSD \downarrow $  & $S_{C} \uparrow$ & $IS \uparrow $ & $FID \downarrow $ \\ \hline
$\lambda$ = 1  & 79.26                           & 67.38      & 4.82           & 14.13             \\ 
$\lambda$ = 10 & \textbf{72.72}                  & 80.28      & {4.59 }          & \textbf{12.30 }            \\
$\lambda$ = 20  &72.86                           & \textbf{ 81.60 }     & \textbf{4.91 }          & 14.73             \\ \hline
\end{tabular} %
\caption{Sensitive analysis for $\lambda$.}
\label{tab:sensitive}
\end{table}

\subsubsection{Sensitivity Analysis of $\lambda$}
We conduct a sensitivity analysis for  $\lambda$ in the $\mathcal{L}_{G}$ in Eq.~(\ref{equ:Gloss}).
As shown in Table~\ref{tab:sensitive}, small $\lambda$ causes degradation of text-image consistency because the optimization step size is too small. We then recommend setting a reasonably large value for $\lambda$ because our SProj can constrain the optimization of $\mathcal{L}_{G}$.

\section{Conclusion}
In this paper, we propose a novel metric $SSD$ to better evaluate text-image consistency. 
Both theoretical analysis and empirical investigation show that  $SSD$ can indeed reflect the semantic consistency in text-to-image generation. We also design a novel framework called PDF-GAN along with two plug-and-play modules that can further enhance the text-image consistency. Experiments on benchmark datasets confirm the effectiveness of $SSD$ as well as the advantages of PDF-GAN both qualitatively and quantitatively.

\begin{figure*}[t]
\begin{minipage}{\textwidth}
 \centering
    \scriptsize
    \includegraphics[width=0.7\linewidth]{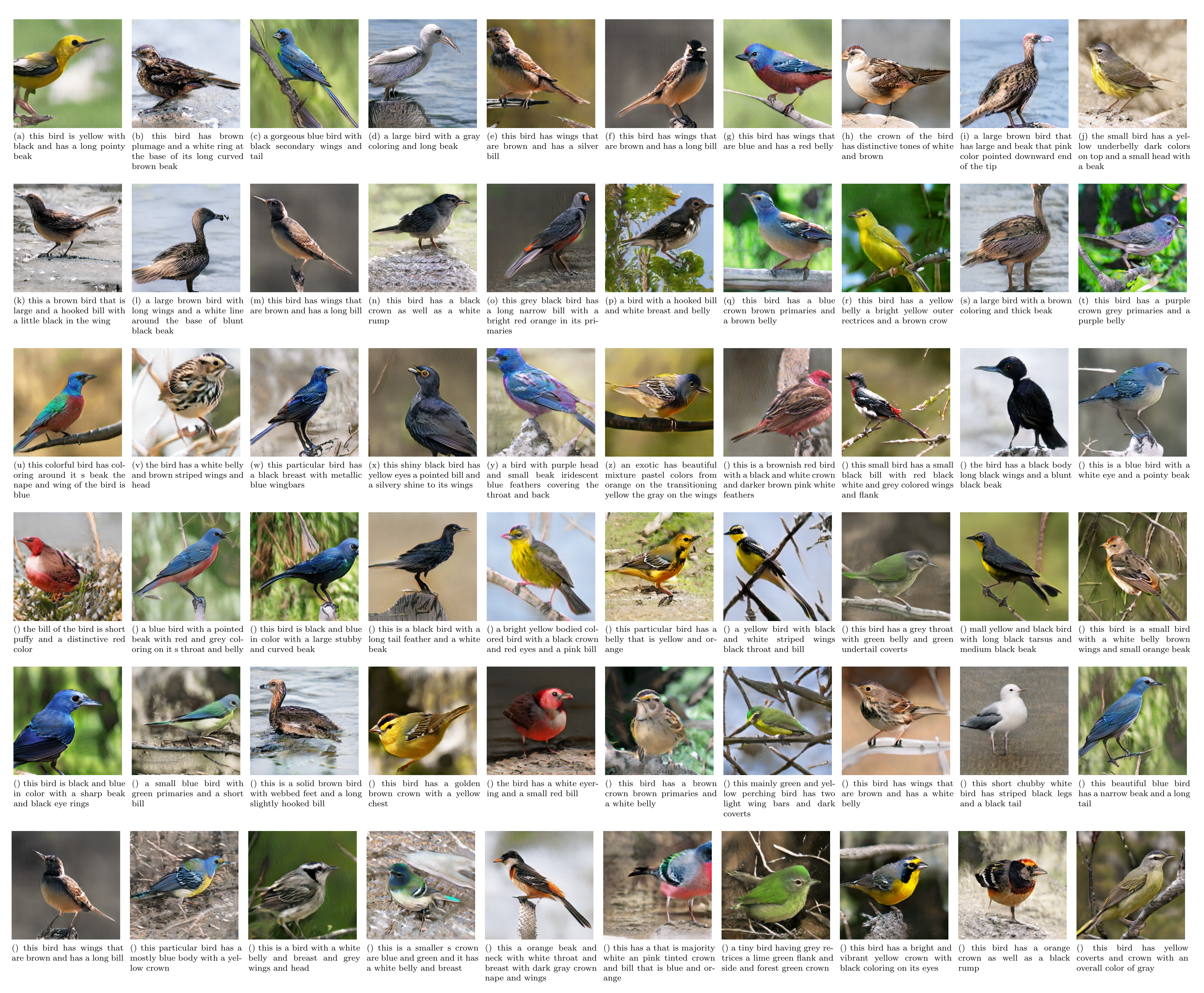}
    \caption{More examples produced by PDF-GAN on CUB.}
    \label{fig:bird_examples_all}
\end{minipage}
\hfill
\begin{minipage}{\textwidth}
  \centering
    \scriptsize
    \includegraphics[width=0.7\linewidth]{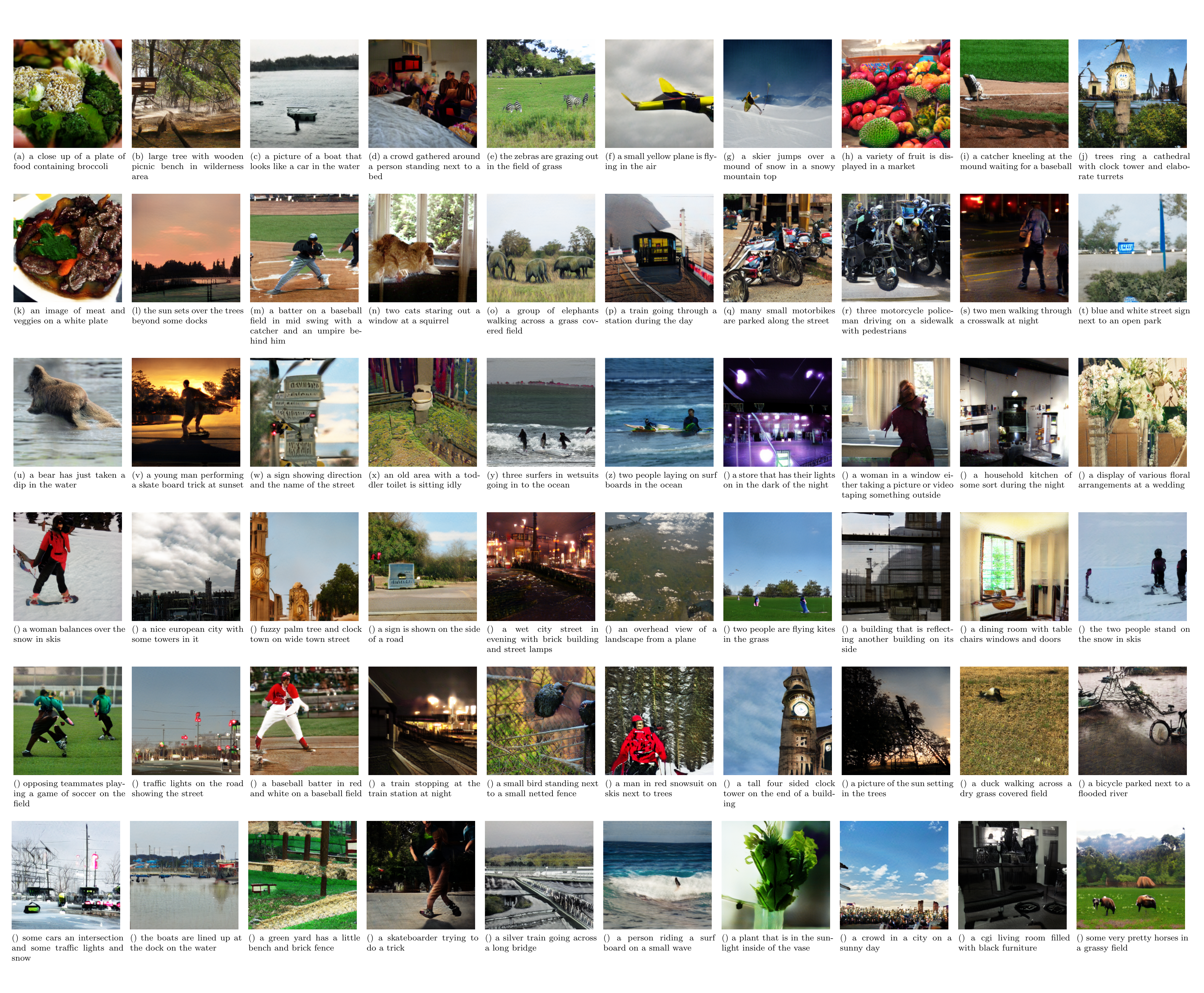}
    \caption{More examples produced by PDF-GAN on COCO.}
    \label{fig:coco_examples_all}
\end{minipage}
\end{figure*}


%
\newpage

\bibliographystyle{IEEEtran}
\bibliography{IEEEtran}

\begin{thebibliography}{10}
\providecommand{\url}[1]{#1}
\csname url@samestyle\endcsname
\providecommand{\newblock}{\relax}
\providecommand{\bibinfo}[2]{#2}
\providecommand{\BIBentrySTDinterwordspacing}{\spaceskip=0pt\relax}
\providecommand{\BIBentryALTinterwordstretchfactor}{4}
\providecommand{\BIBentryALTinterwordspacing}{\spaceskip=\fontdimen2\font plus
\BIBentryALTinterwordstretchfactor\fontdimen3\font minus
  \fontdimen4\font\relax}
\providecommand{\BIBforeignlanguage}[2]{{%
\expandafter\ifx\csname l@#1\endcsname\relax
\typeout{** WARNING: IEEEtran.bst: No hyphenation pattern has been}%
\typeout{** loaded for the language `#1'. Using the pattern for}%
\typeout{** the default language instead.}%
\else
\language=\csname l@#1\endcsname
\fi
#2}}
\providecommand{\BIBdecl}{\relax}
\BIBdecl

\bibitem{yuan2019ckd}
M.~Yuan and Y.~Peng, ``Ckd: Cross-task knowledge distillation for text-to-image
  synthesis,'' \emph{IEEE Transactions on Multimedia}, vol.~22, no.~8, pp.
  1955--1968, 2019.

\bibitem{hong2018inferring}
S.~Hong, D.~Yang, J.~Choi, and H.~Lee, ``Inferring semantic layout for
  hierarchical text-to-image synthesis,'' in \emph{Proceedings of the IEEE
  conference on computer vision and pattern recognition}, 2018, pp. 7986--7994.

\bibitem{li2020exploring}
R.~Li, N.~Wang, F.~Feng, G.~Zhang, and X.~Wang, ``Exploring global and local
  linguistic representations for text-to-image synthesis,'' \emph{IEEE
  Transactions on Multimedia}, vol.~22, no.~12, pp. 3075--3087, 2020.

\bibitem{gou2020segattngan}
Y.~Gou, Q.~Wu, M.~Li, B.~Gong, and M.~Han, ``Segattngan: Text to image
  generation with segmentation attention,'' \emph{arXiv preprint
  arXiv:2005.12444}, 2020.

\bibitem{cheng2020rifegan}
J.~Cheng, F.~Wu, Y.~Tian, L.~Wang, and D.~Tao, ``Rifegan: Rich feature
  generation for text-to-image synthesis from prior knowledge,'' in
  \emph{Proceedings of the IEEE/CVF Conference on Computer Vision and Pattern
  Recognition}, 2020, pp. 10\,911--10\,920.

\bibitem{ramesh2021zero}
A.~Ramesh, M.~Pavlov, G.~Goh, S.~Gray, C.~Voss, A.~Radford, M.~Chen, and
  I.~Sutskever, ``Zero-shot text-to-image generation,'' in \emph{International
  Conference on Machine Learning}.\hskip 1em plus 0.5em minus 0.4em\relax PMLR,
  2021, pp. 8821--8831.

\bibitem{tao2020df}
M.~Tao, H.~Tang, S.~Wu, N.~Sebe, X.-Y. Jing, F.~Wu, and B.~Bao, ``Df-gan: Deep
  fusion generative adversarial networks for text-to-image synthesis,''
  \emph{arXiv preprint arXiv:2008.05865}, 2020.

\bibitem{xu2018attngan}
T.~Xu, P.~Zhang, Q.~Huang, H.~Zhang, Z.~Gan, X.~Huang, and X.~He, ``Attngan:
  Fine-grained text to image generation with attentional generative adversarial
  networks,'' in \emph{Proceedings of the IEEE conference on computer vision
  and pattern recognition}, 2018, pp. 1316--1324.

\bibitem{lin2014microsoft}
T.-Y. Lin, M.~Maire, S.~Belongie, J.~Hays, P.~Perona, D.~Ramanan,
  P.~Doll{\'a}r, and C.~L. Zitnick, ``Microsoft coco: Common objects in
  context,'' in \emph{European conference on computer vision}.\hskip 1em plus
  0.5em minus 0.4em\relax Springer, 2014, pp. 740--755.

\bibitem{wah2011caltech}
C.~Wah, S.~Branson, P.~Welinder, P.~Perona, and S.~Belongie, ``The caltech-ucsd
  birds-200-2011 dataset,'' 2011.

\bibitem{hinz2020semantic}
T.~Hinz, S.~Heinrich, and S.~Wermter, ``Semantic object accuracy for generative
  text-to-image synthesis,'' \emph{IEEE transactions on pattern analysis and
  machine intelligence}, 2020.

\bibitem{hessel2021clipscore}
J.~Hessel, A.~Holtzman, M.~Forbes, R.~L. Bras, and Y.~Choi, ``Clipscore: A
  reference-free evaluation metric for image captioning,'' \emph{arXiv preprint
  arXiv:2104.08718}, 2021.

\bibitem{soloveitchik2021conditional}
M.~Soloveitchik, T.~Diskin, E.~Morin, and A.~Wiesel, ``Conditional frechet
  inception distance,'' \emph{arXiv preprint arXiv:2103.11521}, 2021.

\bibitem{radford2021learning}
A.~Radford, J.~W. Kim, C.~Hallacy, A.~Ramesh, G.~Goh, S.~Agarwal, G.~Sastry,
  A.~Askell, P.~Mishkin, J.~Clark \emph{et~al.}, ``Learning transferable visual
  models from natural language supervision,'' in \emph{International Conference
  on Machine Learning}.\hskip 1em plus 0.5em minus 0.4em\relax PMLR, 2021, pp.
  8748--8763.

\bibitem{liang2022mind}
W.~Liang, Y.~Zhang, Y.~Kwon, S.~Yeung, and J.~Zou, ``Mind the gap:
  Understanding the modality gap in multi-modal contrastive representation
  learning,'' \emph{arXiv preprint arXiv:2203.02053}, 2022.

\bibitem{reed2016generative}
S.~Reed, Z.~Akata, X.~Yan, L.~Logeswaran, B.~Schiele, and H.~Lee, ``Generative
  adversarial text to image synthesis,'' in \emph{International conference on
  machine learning}.\hskip 1em plus 0.5em minus 0.4em\relax PMLR, 2016, pp.
  1060--1069.

\bibitem{zhang2017stackgan}
H.~Zhang, T.~Xu, H.~Li, S.~Zhang, X.~Wang, X.~Huang, and D.~N. Metaxas,
  ``Stackgan: Text to photo-realistic image synthesis with stacked generative
  adversarial networks,'' in \emph{Proceedings of the IEEE international
  conference on computer vision}, 2017, pp. 5907--5915.

\bibitem{zhang2018stackgan++}
------, ``Stackgan++: Realistic image synthesis with stacked generative
  adversarial networks,'' \emph{IEEE transactions on pattern analysis and
  machine intelligence}, vol.~41, no.~8, pp. 1947--1962, 2018.

\bibitem{wang2021text}
M.~Wang, C.~Lang, S.~Feng, T.~Wang, Y.~Jin, and Y.~Li, ``Text to
  photo-realistic image synthesis via chained deep recurrent generative
  adversarial network,'' \emph{Journal of Visual Communication and Image
  Representation}, vol.~74, p. 102955, 2021.

\bibitem{zhu2019dm}
M.~Zhu, P.~Pan, W.~Chen, and Y.~Yang, ``Dm-gan: Dynamic memory generative
  adversarial networks for text-to-image synthesis,'' in \emph{Proceedings of
  the IEEE/CVF Conference on Computer Vision and Pattern Recognition}, 2019,
  pp. 5802--5810.

\bibitem{huang2019realistic}
W.~Huang, R.~Y. Da~Xu, and I.~Oppermann, ``Realistic image generation using
  region-phrase attention,'' in \emph{Asian Conference on Machine
  Learning}.\hskip 1em plus 0.5em minus 0.4em\relax PMLR, 2019, pp. 284--299.

\bibitem{ruan2021dae}
S.~Ruan, Y.~Zhang, K.~Zhang, Y.~Fan, F.~Tang, Q.~Liu, and E.~Chen, ``Dae-gan:
  Dynamic aspect-aware gan for text-to-image synthesis,'' in \emph{Proceedings
  of the IEEE/CVF International Conference on Computer Vision}, 2021, pp.
  13\,960--13\,969.

\bibitem{li2019controllable}
B.~Li, X.~Qi, T.~Lukasiewicz, and P.~Torr, ``Controllable text-to-image
  generation,'' \emph{Advances in Neural Information Processing Systems},
  vol.~32, 2019.

\bibitem{gao2021lightweight}
L.~Gao, D.~Chen, Z.~Zhao, J.~Shao, and H.~T. Shen, ``Lightweight dynamic
  conditional gan with pyramid attention for text-to-image synthesis,''
  \emph{Pattern Recognition}, vol. 110, p. 107384, 2021.

\bibitem{ma2019sd}
J.~Ma, L.~Zhang, and J.~Zhang, ``Sd-gan: Saliency-discriminated gan for remote
  sensing image superresolution,'' \emph{IEEE Geoscience and Remote Sensing
  Letters}, vol.~17, no.~11, pp. 1973--1977, 2019.

\bibitem{zhang2021cross}
H.~Zhang, J.~Y. Koh, J.~Baldridge, H.~Lee, and Y.~Yang, ``Cross-modal
  contrastive learning for text-to-image generation,'' in \emph{Proceedings of
  the IEEE/CVF Conference on Computer Vision and Pattern Recognition}, 2021,
  pp. 833--842.

\bibitem{qiao2019mirrorgan}
T.~Qiao, J.~Zhang, D.~Xu, and D.~Tao, ``Mirrorgan: Learning text-to-image
  generation by redescription,'' in \emph{Proceedings of the IEEE/CVF
  Conference on Computer Vision and Pattern Recognition}, 2019, pp. 1505--1514.

\bibitem{seshadri2021multi}
A.~D. Seshadri and B.~Ravindran, ``Multi-tailed, multi-headed, spatial dynamic
  memory refined text-to-image synthesis,'' \emph{arXiv preprint
  arXiv:2110.08143}, 2021.

\bibitem{dong2021unsupervised}
Y.~Dong, Y.~Zhang, L.~Ma, Z.~Wang, and J.~Luo, ``Unsupervised text-to-image
  synthesis,'' \emph{Pattern Recognition}, vol. 110, p. 107573, 2021.

\bibitem{brock2018large}
A.~Brock, J.~Donahue, and K.~Simonyan, ``Large scale gan training for high
  fidelity natural image synthesis,'' \emph{arXiv preprint arXiv:1809.11096},
  2018.

\bibitem{gal2021stylegan}
R.~Gal, O.~Patashnik, H.~Maron, G.~Chechik, and D.~Cohen-Or, ``Stylegan-nada:
  Clip-guided domain adaptation of image generators,'' \emph{arXiv preprint
  arXiv:2108.00946}, 2021.

\bibitem{wang2022clip}
Z.~Wang, W.~Liu, Q.~He, X.~Wu, and Z.~Yi, ``Clip-gen: Language-free training of
  a text-to-image generator with clip,'' \emph{arXiv preprint
  arXiv:2203.00386}, 2022.

\bibitem{ramesh2022hierarchical}
A.~Ramesh, P.~Dhariwal, A.~Nichol, C.~Chu, and M.~Chen, ``Hierarchical
  text-conditional image generation with clip latents,'' \emph{arXiv preprint
  arXiv:2204.06125}, 2022.

\bibitem{kay1993fundamentals}
S.~M. Kay, \emph{Fundamentals of statistical signal processing: estimation
  theory}.\hskip 1em plus 0.5em minus 0.4em\relax Prentice-Hall, Inc., 1993.

\bibitem{zhang2019self}
H.~Zhang, I.~Goodfellow, D.~Metaxas, and A.~Odena, ``Self-attention generative
  adversarial networks,'' in \emph{International conference on machine
  learning}.\hskip 1em plus 0.5em minus 0.4em\relax PMLR, 2019, pp. 7354--7363.

\bibitem{lopez2017gradient}
D.~Lopez-Paz and M.~Ranzato, ``Gradient episodic memory for continual
  learning,'' \emph{Advances in neural information processing systems},
  vol.~30, 2017.

\bibitem{salimans2016improved}
T.~Salimans, I.~Goodfellow, W.~Zaremba, V.~Cheung, A.~Radford, and X.~Chen,
  ``Improved techniques for training gans,'' \emph{Advances in neural
  information processing systems}, vol.~29, 2016.

\bibitem{heusel2017gans}
M.~Heusel, H.~Ramsauer, T.~Unterthiner, B.~Nessler, and S.~Hochreiter, ``Gans
  trained by a two time-scale update rule converge to a local nash
  equilibrium,'' \emph{Advances in neural information processing systems},
  vol.~30, 2017.

\bibitem{zhang2021dtgan}
Z.~Zhang and L.~Schomaker, ``Dtgan: Dual attention generative adversarial
  networks for text-to-image generation,'' in \emph{2021 International Joint
  Conference on Neural Networks (IJCNN)}.\hskip 1em plus 0.5em minus
  0.4em\relax IEEE, 2021, pp. 1--8.

\end{thebibliography}

\end{document}